\documentclass{article}

\usepackage{arxiv}

\usepackage[utf8]{inputenc} 
\usepackage[T1]{fontenc}    
\usepackage{hyperref}       
\usepackage{url}            
\usepackage{booktabs}       
\usepackage{amsfonts, amssymb}       
\usepackage{nicefrac}       
\usepackage{microtype}      
\usepackage{lipsum}		
\usepackage{graphicx}
\usepackage{natbib}
\usepackage{doi}

\renewcommand{\shorttitle}{ }

\hypersetup{
pdftitle={On the Existence of Universal Simulators of Attention},
pdfauthor={Debanjan Dutta, Faizanuddin Ansari, Anish Chakrabarty, Swagatam Das},
colorlinks=true, citecolor=Aquamarine, linkcolor=magenta
}

\usepackage[dvipsnames]{xcolor}

\usepackage{enumitem}
\usepackage{soul}
\usepackage{xspace,mfirstuc,tabulary, amsfonts, amsmath, amsthm, stmaryrd, listings, caption, subcaption, thmtools, cleveref,rotating, mathtools, fontawesome, pifont, makecell, bbm}
\usepackage[verbose]{wrapfig}
\usepackage{tikz}
\usetikzlibrary{positioning, arrows.meta, shapes.misc, fit, backgrounds, calc}
\usepackage[ruled,linesnumbered]{algorithm2e}
\usepackage[textsize=tiny,textwidth=1.2cm, color=green!10, bordercolor=white]{todonotes}
\newtheorem{theorem}{Theorem} 
\newtheorem{corollary}{Corollary}[theorem]
\newtheorem{lemma}[theorem]{Lemma}
\newtheorem{proposition}[theorem]{Proposition}
\theoremstyle{remark}
\newtheorem{remark}{Remark}[theorem]
\makeatletter
\def\th@remark{
  \thm@headfont{\bfseries\itshape} 
  \thm@notefont{\normalfont\itshape}         
  \normalfont                      
}
\makeatother

\definecolor{lightgray}{rgb}{.9,.9,.9}
\definecolor{darkgray}{rgb}{.4,.4,.4}
\definecolor{purple}{rgb}{0.65, 0.12, 0.82}
\newcommand{\cmark}{\ding{51}}
\newcommand{\xmark}{\ding{55}}

\lstdefinelanguage{rasp}{
  keywords={def, return, if, else},
  keywordstyle=\color{Apricot}\bfseries,
  keywords=[2]{indices, tokens, full\_s, tokens\_int, tokens\_float},
  keywordstyle=[2]\color{Violet}\bfseries,
  keywords=[3]{length, full\_s, selector\_width},
  keywordstyle=[3]\color{Fuchsia}\bfseries,
  keywords=[4]{select, aggregate},
  keywordstyle=[4]\color{Green}\bfseries,
  keywords=[5]{set, examples, example, off},
  keywordstyle=[5]\color{Magenta}\bfseries,
  otherkeywords={==,<=,>=,<,>},
  morekeywords=[6]{==,<=,>=,<,>, and, or, not},
  keywordstyle=[6]\color{cyan}\bfseries,
  identifierstyle=\color{black},
  sensitive=false,
  comment=[l]{\#},
  morecomment=[s]{'''}{'''},
  commentstyle=\color{purple}\ttfamily,
  stringstyle=\color{red}\ttfamily,
  morestring=[b]',
  morestring=[b]"
}
\makeatletter
\lst@AddToHook{OnEmptyLine}{\addtocounter{lstnumber}{-1}}
\makeatother

\title{ On the Existence of Universal Simulators of Attention } 

\author{
 Debanjan Dutta \\
 Indian Statistical Institute \\
 West Bengal, India \\
\And
 Anish Chakrabarty \\
 Indian Statistical Institute \\
 West Bengal, India \\
\And
 Faizanuddin Ansari \\
 Indian Statistical Institute \\
 West Bengal, India \\
\And
 Swagatam Das \\
 Indian Statistical Institute \\
 West Bengal, India \\
}

\renewcommand{\vec}[1]{\mathrm{\mathbf{#1}}}
\renewcommand{\mod}{\operatorname{mod}}
\DeclareMathOperator{\Pool}{Pool}

\DeclareMathOperator{\f}{\circledcirc}
\DeclareMathOperator{\concat}{\odot}
\DeclareMathOperator{\prop}{prop}
\renewcommand{\vec}[1]{\mathrm{\mathbf{#1}}}
\newcommand{\N}{\mathbb{N}}
\newcommand{\R}{\mathbb{R}}
\newcommand{\softmax}{\mathrm{softmax}}

\begin{document}
\maketitle
\begin{abstract}
    Previous work on the learnability of transformers \textemdash\ focused primarily on examining their ability to approximate specific algorithmic patterns through training \textemdash\ has largely been data-driven, offering only probabilistic guarantees rather than deterministic solutions. Expressivity, on the contrary, has been devised to address the problems \emph{computable} by such architecture theoretically. These results proved the Turing-completeness of transformers, investigated bounds focused on circuit complexity, and formal logic. Being at the crossroad between learnability and expressivity, the question remains: \emph{can a transformer, as a computational model, simulate an arbitrary attention mechanism, or in particular, the underlying operations?} In this study, we investigate the transformer encoder's ability to simulate a vanilla attention mechanism. By constructing a universal simulator $\mathcal{U}$ composed of transformer encoders, we present algorithmic solutions to replicate attention outputs and the underlying elementary matrix and activation operations via RASP, a formal framework for transformer computation. We show the existence of an algorithmically achievable, data-agnostic solution, previously known to be approximated only by learning.
\end{abstract}

\section{Introduction}
\label{sec:1}
The widespread adoption of language models across diverse fields of study, whether in task-specific applications \citep{lin2022survey, haruna2025exploring, consens2025transformers} or theoretical verifications \citep{strobl2024formal}, underscores the remarkable success of attention-based transformers. These models have demonstrated the ability to \emph{learn} from tasks and function as \emph{simulators} of a broad range of computational architectures. While ongoing investigations seek to characterize the representational power of trained transformers from both statistical (in-context \citep{mroueh2023towards, kim2024transformers}) and computational perspective \citep{merrill-etal-2020-formal, liu2023transformers, merrill2024the}, a fundamental question remains unanswered: \textit{irrespective of the complexity class to which a transformer belongs, can a mechanism simulate attention itself using only interactions between vanilla transformers?} Specifically, we ask whether such a mechanism exists that emulates the functioning of a single-layer transformer encoder, given we have access to a system with transformers as the only computational model. Throughout our discussion, we focus on the self-attention mechanism of the transformer encoder. Note that these encoders can be solely characterized by their parameters.

To put the problem into perspective, we highlight that theoretical analyses of transformers often involve \emph{hard} (\emph{unique} or \emph{average}) attentions. The language \textsf{PARITY}=$\{w\in\{0,1\}^* \mid \#_1(w) \mod 2 = 0\}$ in particular, has been contextual in most of the investigations. \cite{hahn-2020-theoretical} pointed out the inability of transformers to recognize the language. Although the learnability of such transformers was not amenable \citep{bhattamishra2020ability}, \citeauthor{chiang2022overcoming}'s approach to overcoming this drawback entailed the explicit construction of a multi-layer multi-head $\softmax$ attention transformer (SMAT). The language, $k\text{-\textsf{PARITY}} = \{w\in\{0,1\}^n \mid S \subset \{0,1,\dots,n-1\} \text{ and } (\sum_{i_j \in S} w_{i_j}) \mod 2 = 0 \}$ where $|S|=k \ll n$, has been shown single-layer multi-head SMAT-learnable by \citeauthor{han2025attention}. To further detail the representational power of such SMATs, we mention the pair detection task $\text{Match}_2$, defined as $\text{Match}_2(S)_{i \in [|S|]} = \mathbbm{1}\{\exists j \mid (s_i+s_j) \mod p = 0\}$, where $S = (s_1,s_2,\dots,s_{|S|}) \in \{1,2,\ldots,p\}^{|S|}\}$ and $p$ is very large \citep{sanford2023representational}. While the same can be solved using a single-layer single-head SMAT, $\text{Match}_3$, an extension with three variables does not follow suit, even with the multi-layer multi-head extension. It may be worth mentioning here that even though a simple average hard attention transformer (AHAT) can be realized to compute the task $\text{Match}_2$ (see \autoref{lst:Match2} for a RASP implementation), the goal of this work is to go beyond prescribing task-specific AHATs. In specific terms, we aim to propose a universal simulator $\mathcal{U}$ composed of transformer encoders that can replicate any single-layer multi-head SMATs, including the one designed exclusively for $\text{Match}_2$. Our construction consolidates both notions of attention (hard and soft) into a unified computational model (namely $\mathcal{U}$) capable of performing any single-layer multi-head attention. As a stepping stone, our investigation simulates single-layer multi-head SMATs using (A)HATs, in line with \cite{yang2024simulating}. Our formulations are supported by Restricted Access Sequence Processing (RASP) \citep{weiss21a} implementations: a formal, human-readable framework that models transformer computation with parallel, attention-driven processing while enforcing constraints such as fixed computation depth, elementwise operations, and pairwise token dependencies. Specifically, given a single-layer transformer-attention $T$ and an input $X$, the $\mathcal{U}$ we construct is shown to realize the output (say, $T(X)$), on receiving the pair $\langle T, X\rangle$ (see \autoref{fig:U}). 

\begin{wrapfigure}[22]{r}{0.46\textwidth}
    \centering
    \vspace{-0.5\intextsep}
    \tikzset{
    rect/.style={draw=none, rounded corners=10pt, fill=orange!12, inner sep=8pt},
    smallbox/.style={draw=none, rounded corners=8pt, fill=green!10, inner sep=0pt},
    circle_node/.style={circle, draw=none, fill=blue!10, minimum size=18pt, inner sep=1pt},
    arrow/.style={-{Latex}, thin, draw=black!70},
    input_label/.style={font=\itshape}
}
\resizebox{0.45\textwidth}{!}{
\begin{tikzpicture}[
    node distance=1.5cm and 1.5cm,
    every node/.style={font=\sffamily},
]

    \node[input_label] (X) at (0,3.5) {$X$};
    \node[input_label] (A) at (0,2.5) {$A$};
    \node[input_label] (T) at (-0.65,1.75) {$T$};
    \node[input_label] (V) at (0,1) {$V$};

    \node[input_label] (padding) at (1.2,0) {};

    \node[circle_node, right=1.8cm of X] (m1) {$\otimes$};
    \node[circle_node, below=.75cm of m1] (top) {$\top$};
    \node[circle_node, below=.35cm of top] (m2) {$\otimes$};

    \path ($(m1)!0.5!(top) + (1.4cm,0)$) node[circle_node] (m3) {$\otimes$};

    \node[circle_node, right=0.85cm of m3] (sigma) {$\sigma$};
    \node[circle_node, right=0.85cm of sigma] (m4) {$\otimes$};

    \begin{pgfonlayer}{background}
        \node[rect, fit=(m1)(top)(m2)(m3)(sigma)(m4)(padding), label=above:{\Large$\mathcal{U}$}] {};
        \node[smallbox, fit=(A)(T)(V)] {};
    \end{pgfonlayer}

    \node[circle_node, right=.2cm of T, yshift=-2.6cm, label=right:{\textnormal{: Lemma \ref{lem:transposition}}}] (legtop) {$\top$};
    \node[circle_node, right=1.9cm of legtop, label=right:{\textnormal{: Lemma \ref{lem:softmax}}}] (legsigmoid) {$\sigma$};
    \node[circle_node, right=1.9cm of legsigmoid, label=right:{\textnormal{: Lemma \ref{lem:matmul}}}] (legmatmul) {$\otimes$};

    \draw[arrow] (X) -- (m1);
    \draw[arrow] (A) -| (m1);
    \draw[arrow] (V) -- ++(1.2,0) |- ++(0,-.8) -- ++(1.2,0) -- (m2);
    \draw[arrow] (X) -- ++(1.2,0) |- (top);
    \draw[arrow] (X) -- ++(1.2,0) |- (m2);

    \draw[arrow] (m1) -- ++(.65,0) |- ++(0,-0.7) -- (m3);
    \draw[arrow] (top) -| (m3);
    \draw[arrow] (m3) -- (sigma);
    \draw[arrow] (sigma) -- (m4);
    \draw[arrow] (m2) -| (m4);
    \draw[arrow] (m4) -- ++(1,0);

\end{tikzpicture}
}
    \caption{Simulation of attention $T$ characterized by matrices $A$ and $V$ on input $X$ using the proposed transformer network $\mathcal{U}$ such that $\mathcal{U}\left(\langle T, X\rangle\right) = T(X)$. Operations $\top,  \boxtimes \text{ and } \sigma$ represent matrix \textit{transposition}, \textit{multiplication} and activation $\softmax$ implemented using transformer as presented in \autoref{lem:transposition}, \autoref{lem:matmul} and \autoref{lem:softmax} respectively.}
    \label{fig:U}
\end{wrapfigure}

Our approach in this study analogizes the rationale à la Universal Turing Machine (UTM). Observe that a UTM $U$ accepts the encoded pair $\langle \hat{T}, w\rangle$ if and only if the Turing machine $\hat{T}$ accepts the word $w$. Inspired by the same, \cite{kudlekexistence} explored the (non)-existence of such universal automata for some weaker classes, such as finite and pushdown automata. This, in turn, motivates the inquiry into the simulation of other computational models. Analogously, our constructed transformer network, $\mathcal{U}$, when deemed a language recognizer, can either accept or reject depending upon whether the original transformer encoder $T$ accepts or rejects $X$. When viewed as a \emph{transducer}, it can produce the same output as the original transformer encoder, $T$, on input $X$. It is rather natural to explore the idea of such self-simulation for architectures coupled with decoder attention, given the Turing-completeness \citep{perez2021attention}. Ours, in contrast, involves encoder-only architectures, which have further limited computational capabilities \citep{strobl2024formal}. As defined in \cite{hao2022formal}, the self-attention mechanism introduced by \citeauthor{vaswani2017attention} belongs to the category of \emph{restricted} transformers. Our simulations will be confined to this class of transformers due to their ubiquitous influence.

Our approach also bridges a crucial gap between expressivity and learnability of transformer models. The problem $k$-\textsf{PARITY}, for example, achieves learnability through transformers \citep{han2025attention}. On the other hand, our construction provides not only a definitive method to solve the same, but its applicability can also be generalized for related problems, e.g., $\text{Match}_2$. To contextualize, we point to the long line of works that explore the expressiveness of transformers in simulating important models of computation \citep{perez2019turing, perez2021attention, hao2022formal, barcelo2024logical}, without determining the exact computational classes that include and are included by a transformer's recognition capacity. On the other hand, guarantees regarding the learning capacity of theoretically constructed transformers and their verification toward generalization onto the learned computation procedure (e.g., gradient descent in function space \citep{cheng2024transformers}; Newton's method updates in logistic regression \citep{giannou2024transformersimulate}) inherently become probabilistic and data-dependent. Such results lose justification in scenarios where approximation errors are unacceptable (e.g., formal verification). In contrast, our proofs provide a solution that \textit{algorithmically enforces} correct attention behavior, ensuring reliability beyond data-driven approximations. From a probabilistic viewpoint, this can be regarded as an approximation guarantee with \textit{certainty}, i.e., with $\mathbb{P}$-measure $1$, given $X$ follows the law $\mathbb{P}$.

\noindent\textbf{Contributions.} The key takeaways of our study are as follows. 
\begin{itemize}[leftmargin=*, itemsep=1pt]
    \item We introduce a novel construction framework of amenable matrix operations underlying attention, such as transposition (\autoref{lem:transposition}), multiplication (\autoref{lem:matmul}), determinant calculation, and inversion (\autoref{inversion}) using a transformer itself. We also show that algorithmic constructions exist that exactly represent activation outputs (e.g., $\operatorname{softmax}$ (\autoref{lem:softmax}), $\operatorname{MaxMin}$ (\autoref{lem:MaxMin})). The results combined present a new approach to proving a transformer encoder's expressivity towards a Lipschitz continuous function. All of our constructions come with corresponding RASP implementations and are available in the following repository {\small\url{https://anonymous.4open.science/r/TMA}}.
    \item Our proposed simulator network $\mathcal{U}$ maintains parity with the architectures under simulation in terms of the following fundamental architectural resemblance. Due to its sole reliance on the number of input symbols, $\mathcal{U}$ possesses an inherent hierarchy while expressing attentions of increasing order, leading to a universal simulator (\autoref{Th:1}, \ref{thm:U_multihead}, \autoref{col:expressive_hierarchy}).
    \item Given architectural specifications, our construction, for the first time, ensures the feasibility of simulating soft-attention (restricted transformer) using average-hard attention (RASP) (\autoref{rem:ahattosmat}, \ref{rem:multihead}). The result extends to models involving multiple heads as long as their aggregation mechanism satisfies AHAT-interpretability (\autoref{rem:linformer}).
\end{itemize}

\section{Related Works}
\label{sec:relworks}

It is crucial to contextualize our use of the term \emph{simulate}. Every representation or simulation achieved within classical computational frameworks (namely, the von Neumann architectures that fundamentally drive modern computation) inherently relies on finite-precision approximations when computing over the reals. Nevertheless, such approximations do not compromise the validity of the underlying operations. We adopt a functionally equivalent paradigm in our approach. Specifically, if any numerical operation is represented by the FFNs of the simulating transformer (which is also a supported property in RASP), we invoke the universal approximation theorem \cite{hornik1989multilayer} to assume it can be approximated to any arbitrary precision $\epsilon > 0$. This ensures that the simulated representation remains as robust and functionally indistinguishable as its original output. By explicitly formulating the transformations required to simulate matrix operations, including transposition, multiplication, and inversion within the constraints of a transformer, our work provides a novel foundation for understanding the representational capacity of self-attention. This marks an improvement over \citet{giannou2023looped}'s construction, which relies on a computational framework that is not entirely transformer-based and amplifies input size, e.g., transposing a $d \times d$ matrix requires an $O(4d) \times d^2$ input. We provide a detailed walkthrough of the same in   \autoref{tab:comparison}.

\textbf{Simulation of computational models via transformers.}
Existing work in this line lacks uniformity in the transformer architectures they consider. The distinction primarily stems from the presence (or rather the lack) of decoders and the specific implementation of positional encoding. Introduced by \citet{perez2019turing} and later \cite{hahn-2020-theoretical}, much attention has been given to studying the theoretical capabilities of transformers, characterizing their expressivity in terms of diverse circuit families \citep{hao2022formal, merrill2022saturated, chiang:2025}. 
Along this line, the development of domain-specific languages (DSLs), such as RASP \citep{weiss21a}, enabling the expression of self-attention and transformer operations in a human-interpretable way, inspired several investigations \citep{strobl2024transformers} {--} we review B-RASP \citep{yang2024masked} and C-RASP \citep{yang2024counting}, variants those explore the effect of masking on transformer encoders while the former is restricted to boolean values; RASP-L \citep{zhou2024what}, a learnable version considering tasks of generalized length on a decoder-only autoregressive version of transformers. These DSLs lead to different expressiveness, based on both the augmentation and constraining of the features; however, could hardly challenge the original framework proposed in RASP as a true computational model of transformer encoders.
It is aptly contextual in our discussion, given our goal of prescribing realizable pathways to simulate operations underlying attention. 
On the other hand, \citet{yang2024simulating} demonstrated the realization of hard attention through soft attention. The proof involves simulating a logical language family, implementable by both mechanisms. Studies using a transformer as a language recognizer have also been pursued. Backed by the empirical studies \citet{dehghani2019universal, Shi_Gao_Tian_Chen_Zhao_2022, deletang2023neural}, the expressivity of transformers has been investigated by measuring their equivalence with Turing machines \citep{bhat2020computational, perez2021attention}. Remarkably, \citeauthor{bhat2020computational} simulated a recurrent neural network (RNN), which is a Turing-complete architecture \citep{Hava1992}, via transformers. Later, \citeauthor{perez2021attention} used transformers to carry out a direct simulation of a Turing machine. More recently, \citet{merrill2023logic, barcelo2024logical} drew equivalence with the logical expressions accepted by transformers. In the broader landscape of such remarkable findings, our question remains ever so practical and fundamental: can a transformer, when viewed as a computational model, simulate the basic constituent operations and the attention mechanism? 

\begin{table}[t]
\centering
\caption{Comparative analysis between the softmax-based transformer architecture proposed by \cite{giannou2023looped} (\textit{first blocked row}) and our proposed AHAT-based architecture (\textit{second blocked row}). We report number of layers ($L$), head count ($H$) (\textit{first row}), and vectorized input complexity (\textit{second row}). \textdagger\ indicates the maximum of input complexity required for multiplication, subtraction, and transposition. `Dep' indicates if weights depend on input values (\textit{third row}). For multiplication, on matrices $A \in \R^{r\times k}$ (vs. $A \in \R^{k\times r}$) and $B\in \R^{k\times c}$, we compute $AB$ (vs. $A^\top B$ where $d > \max(k,r,c)$). Also, in case of inversion, we perform exact inversion of $3 \times 3$ matrices in contrast to \citet{giannou2023looped} implementing Newton's method of inversion, which fundamentally differs from the transformer architecture by considering loop as an additional property.}
\begin{tabular}{p{2.3cm}ccc}
\toprule
 & Transposition & Multiplication & Inversion \\
\midrule
\cite{giannou2023looped} & \makecell[t]{$4L, 1H$ \\ $O(4r^3)$ \\ Dep: \cmark \\ \scriptsize{(Square only)}} 
                                  & \makecell[t]{$2L, 1H$ \\ $O(54d^2)$ \\ Dep: \cmark \\ \scriptsize{(Computes $A^\top B$)}} 
                                  & \makecell[t]{$13L, 1H$ \\ \textdagger \\ Dep: \cmark \\ \scriptsize{(\textbf{Any})}} \\
\midrule
Ours                     & \makecell[t]{$2L, 1H$ \\ $O(r\cdot)$ \\ Dep: \cmark \\ \scriptsize{(\textbf{Any})}} 
                                  & \makecell[t]{$3L, O(rc)H$ \\ $O(k(r+c))$ \\ Dep: \cmark \\ \scriptsize{(\textbf{Computes $AB$})}} 
                                  & \makecell[t]{$5L, 4H$ \\ $O(r^2)$ \\ Dep: \cmark \\ \scriptsize{($3 \times 3$ matrix)}} \\
\bottomrule
\end{tabular} \vspace*{-0.68em}
\label{tab:comparison}
\end{table}

\textbf{Approximation and learnability.} It is worth mentioning that most proofs regarding the simulation of any system (above discussion) do not provide a realizable pathway for adoption in a learning setup. Our proofs, in contrast, are designed with the very goal of realizability in mind. Building towards it, let us walk through the accrued wisdom in this context. While vanilla transformer encoders are, in general, universal approximators of continuous sequence-to-sequence (permutation equivariant) maps supported on a compact domain \citep{Yun2020Are}, they require careful construction to extend the property to models with nonlinear attention mechanisms \citep{alberti2023sumformer} and non-trivial positional encoding \citep{luo2022your}. However, it remains unclear whether the approximation capability holds while learning, given the unidentifiability of additional optimization errors due to data-driven training. In this context, we also mention that transformers can learn sparse Boolean functions of input samples having small bounded weight norms \citep{edelman2022inductive}. Along the line, \citet{yau2024learning} ensures that multi-head linear encoders can be learned in polynomial time under $L^2$ loss. Corroborating \citet{perez2021attention}'s finding in a learning setup, \citet{wei2021advances} also show that output classes of functions from TMs can be efficiently approximated using transformers (encoder-decoder). In contrast, the domain that has received the most attention is transformers' capacity to learn tasks in-context (IC) \citep{mroueh2023towards}. Under varying assumptions on the architecture and data, transformers provably tend to emulate gradient updation \citep{ahn2023transformers, cheng2024transformers}, Newton's iterations \citep{giannou2024transformersimulate}, and perform linear or functional regression \citep{fu2024transformers,pathak2024transformers,zhang2024trained}. 

Perhaps the work that is most aligned to our line of inquiry is \citet{giannou2023looped}. However, we emphasize that their \texttt{SUBLEQ} or the generalized \texttt{FLEQ} designs augment inputs with scratchpad, memory and instruction, and the output often includes non-essential residuals (e.g., duplicate results of a matrix transposition, Lemma 12) unless post-processed. Even though their overall layer-count and head-count are constant, several hyperparameters lie intrinsically dependent on the number of tokens ($n$), e.g., the approximation bound is valid when temperature $\lambda \ge \log \frac{n^3}{\epsilon}$ (Lemma 2). Similarly, the assignment of the non-trivial parameters $V$ depends on $n$ (Lemma 13). Above all, its underlying computational framework fundamentally increases the depth (i.e., the layer count) by allowing loops. In \autoref{tab:comparison}, we outline the key differences that distinguish our work from \cite{giannou2023looped}. Therefore, it is justified to infer that the ensuing computational power of the model becomes stronger compared to self-attention-based transformers \citep{hahn-2020-theoretical, feng2023towards, qiu2025ask}. From a learning point of view, our realizable algorithmic pathways can be deemed a solution that provides certainty for simulating fundamental operations underlying the attention mechanism.

\section{Preliminaries}
\label{sec:prelim}
An element in $n$-dimensional array $M$ has been represented using $M[i_0,i_1,\ldots,i_{n-1}]$. To reduce notational overhead, we denote by $M[i_0]$ the induced $(n-1)$-dimensional array associated with the $i_0$\textsuperscript{th} index of the introductory dimension in $M$. Note that a matrix is a two-dimensional array. We also highlight the difference between the usage of $\concat\limits_{i=1}^{m}\left(a_i\right)$ and $\f\limits_{i=1}^{n}\left(a_i\right)$. While the former, the concatenation operation, denotes the expression $a_1\concat a_2 \concat \ldots \concat a_m$ as usual, the latter, an $n$-ary operation, is used to denote $\f\left(a_1, a_2, \ldots, a_n\right)$. Given the sets of sequences $\mathcal{S}$ and $\mathcal{S}^\prime$, we define a mapping $f: \mathcal{S} \to \mathcal{S}^\prime$ as \emph{length-preserving} if for any $S\in \mathcal{S}$, $|S| = |f(S)|$, where $|\cdot|$ implies the number of elements in the sequence. Given sequences $S_1, S_2 \in \mathbb{R}^n$, the operations $/$ and $\otimes$ respectively induce elementwise division and multiplication, i.e., 
$(S_1 / S_2)[i] \coloneqq \frac{{S_1}[i]}{{S_2}_[i]}$ and $(S_1 \otimes S_2)[i] \coloneqq {S_1}[i] \cdot {S_2}[i]$. Additionally, the operation $\circ$ on a matrix $R \in \R^{r\times c}$ and a sequence $S \in \R^{r}$ produces another sequence $S^\prime \in \R^{c}$ such that $S^\prime[i] = \sum_{j\in [c-1]_{\cup\{0\}}}R[i, j]\cdot S[j]$. Note that taking $r = c$ produces a length-preserving sequence. Keeping to convention, $\mathbf{1}^{r\times c}$ and $\mathbf{0}^{r\times c}$ denote matrices with all entries $1$ and $0$ respectively. And $[r]_{\cup\{0\}}$ denotes the set $\{0,1, \ldots, r\}$. $\Pi([r])$ denotes the set of all permutations of $[r]$.

\subsection{Transformer Encoder}
\label{ssec:transenc}
A transformer encoder is a layered neural network that maps strings to strings. The input layer maps the string to a sequence of vectors. The subsequent layers apply the attention mechanism, which is composed of the sublayers' self-attention and feed-forward components. For ease of representation, we avoid the layer normalization mechanism. The final output layer maps the sequence of vectors back to a string. The following discussion formalizes the same.\\[1.5pt]
\textbf{Input.} Let $\vec{w} = w_0w_1\dots w_{n-1}$ be a string, where each character $w_i$ belongs to the input alphabet $\Sigma$. We assume the input layer of any transformer to be composed of the word embedding $\textrm{WE}:\Sigma\to \R^d$ and positional embedding $\textrm{PE}:(\N\times \N)\to \R^d$ in an additive form, so that the produced input vector becomes $X = (\vec{x}_0,\vec{x}_1,\ldots,\vec{x}_{n-1}) \in \R^{n\times d}$ such that $\vec{x}_i = \textrm{WE}(w_i) + \textrm{PE}(i, |w|)$.\\[1.5pt]
\textbf{Encoder attention.} The first component of an encoder layer $\ell$ is self-attention. Assuming $X^{(0)} = X$, on an input $X^{(\ell-1)}, \ell \in \{1,2,\ldots, L\}$ a self-attention mechanism produces
\begin{align}
\label{eq:softmaxatt}
    \sigma\Big(X{^{(\ell-1)}}W_Q{^{(\ell)}}{W_K^{(\ell)}}^{\top}{X^{(\ell-1)}}^{\top}\Big)X^{(\ell-1)}W_V^{(\ell)},
\end{align}
where, $\sigma$ is a $\softmax$ activation, computing the attention scores from the query $X^{(\ell-1)}W_Q^{(\ell)}$ and key $X^{(\ell-1)}W_K^{(\ell)}$ to draw the influential value vectors $X^{(\ell-1)}W_V^{(\ell)}$ in a composite form, where the weight matrices $W_Q^{(\ell)}, W_K^{(\ell)} \in \mathbb{R}^{d\times d^\prime}, W_V^{(\ell)}\in \mathbb{R}^{d\times d_{v}} \text{ and the input } X \in \mathbb{R}^{n\times d}$. Following a residual connection, a subsequent feed-forward network (FFN), consisting of two linear transformations with a $\operatorname{ReLU}$ activation in between, is applied to this result. Note that in the above expression the projections $W_Q^{(\ell)}$ and $W_K^{(\ell)}$ can be combined to result $A^{(\ell)} \eqqcolon W_Q{^{(\ell)}}{W_K^{(\ell)}}^{\top} \in \R^{d\times d}$, and to simplify notations, we rename $W_V$ to $V$. A self-attention at a layer $\ell$ can thus be uniquely characterized by the three parameters $A^{(\ell)}, V^{(\ell)}$ and any normalizing activation function, here taken as $\softmax$. We will drop the notation $\ell$ wherever the context is self-explanatory.
We call a transformer attention $T$ applied on input $X\in\R^{n\times d}$ of order $(n, d, d_v)$ if its characterizing matrices $A\in\R^{d\times d}$ and $V\in \R^{d\times d_v}$. Similarly, when a single-layer transformer $T$ with parameters $W_1\in\R^{d_{v}\times d_{1}}$ and $W_2\in \R^{d_{1}\times d_{2}}$ in feed-forward sublayer is applied on input $X\in\R^{n\times d}$, we call it of order $(n, d, d_v, d_1, d_2)$.

\subsection{GAHAT}
\label{ssec:gahat_rasp_overview}
A generalized attention, as proposed by \citet{hao2022formal}, takes the query and key as input and does not restrict them to be combined using the dot-product operation only. Instead, any computable association can be employed to calculate attention scores. Finding the dominant value vectors has also been kept flexible using a function $\Pool$ that takes the value vectors and the attention scores. When this function is particularly \emph{unique} (or, \emph{average}) hard, such transformers are regarded as generalized unique (or, average) hard attention transformers (GUHAT or GAHAT). As such, given value vectors $XV = (\vec{y}_0,\vec{y}_1,\ldots,\vec{y}_{n-1})$ and attention scores $(a_0, a_1, \ldots, a_{n-1})$, let $j_0, j_1,\ldots,j_{m-1} \in [n-1]_{\cup \{0\}}$ are the indices in ascending order such that they maximize $a_j$s. Then, unique hard attention pools $\vec{y}_{j_0}$ while average hard attention pools $\frac{1}{m}\sum_{i=0}^{m-1}\vec{y}_{j_i}$ (see also \citet{merrill2022saturated}).

The computational model underlying the Restricted Access Sequence Processing Language (RASP), introduced by \citet{weiss21a}, resembles that of GAHAT, based on overlapping sufficiency characterizations (Section 3.1 in \citet{weiss21a} and Section 4.3 in \citet{hao2022formal}). Also, in the same light, the expression above (\ref{eq:softmaxatt}) falls under the category of \textit{restricted} transformers. RASP is a human-interpretable, sequence-processing DSL for designing transformer encoders. It operates on a sequence of tokens (e.g., characters, numbers, Booleans) to produce a length-preserved output sequence. Its core syntax includes elementwise operations and two non-elementwise operations: \lstinline{select} and \lstinline{aggregate}, which together correspond to a single self-attention layer. Token values and positions are accessed via \lstinline{tokens} and \lstinline{indices}. Lacking loops, RASP execution is inherently parallelizable operations, mirroring self-attention (via \lstinline{select} and \lstinline{aggregate} pair that resembles the $QKV$ operation), with elementwise operations reflecting terminal feed-forward layers. This absence of iterative constructs limits its applicability to inherently sequential computations, a direct consequence of the transformer's constant-depth nature that prevents arbitrary iteration simulation in one pass. Note that the \lstinline{aggregate} operation is crucial for derived operations like \lstinline{length}, which returns a sequence of the scalar repeated to maintain length.

Given the definition of Average Hard Attention (AHA) by \citet{hao2022formal} (Def. 9), and the fact that \lstinline{aggregate} performs an \emph{average} over value vectors from the Boolean attention matrix generated by \lstinline{select}, it is evident that the attention module in RASP is AHA. The \lstinline{select} operation uses a Boolean predicate to associate keys and queries, placing it under the category of Generalized Average Hard Attention (GAHA). While GAHAT allows any terminating aggregator function\footnote{The choice of the codomain of $g$ as $\{0, 1\}$ is purely based on the objective of language recognition.} (which is a ReLU-activated FFN for restricted transformers), RASP permits any FFN for the same. The only sufficient condition for an activation to be compliant with RASP is universal (also uniform) approximation with arbitrary accuracy of regular maps, e.g., continuous Borel-measurable functions, Besov functions, etc.

In the scope of restricted transformers, various attention mechanisms have been employed to achieve faster computation, differing mainly in their choice of characterizing matrices and/or the $\Pool$ function. For instance, Linformer \citep{wang2020linformer} is one that introduces new characterizing matrices $E, F \in \mathbb{R}^{k\times n}$ for some $k < n$ such that the attention becomes
\begin{align}\label{eq:linformer}\sigma\left((XW_Q)\left(EXW_K\right)^{\top}\right)FXV.\end{align}
A linear attention \citep{LinearTransformer}, on the other hand, assumes no $\Pool$, resulting in:
\begin{align}\label{eq:linatt}(XW_Q)\left(XW_K\right)^{\top}XV.\end{align}

\section{On Simulating Attention}
\label{sec:simulation}
\vspace{-2pt}
In this section, we will provide all necessary propositions and lemmas towards constructing the transformer network $\mathcal{U}$ that provably simulates arbitrary transformer attention $T$ of order $(n, d, d_v)$ (\autoref{Th:1}). 
\autoref{lem:transposition}-\ref{lem:matmul} use the representation given in \autoref{1D} to perform some basic operations such as matrix transposition, applying $\softmax$ activation, and matrix multiplication.
\autoref{algo:trans}-\ref{algo:matmul} serve the constructive proofs of the respective lemmas via GAHAT. \autoref{lst:transpose}-\ref{lst:matmul} provide the corresponding RASP codes. Notice that, while these algorithms involve notation $r$ denoting the matrix order (i.e. the number of rows), we have considered $r=3$ in the RASP codes as presented in Appendix \ref{ssec:RASPCodes}.

Before proceeding with the main construction, let us establish some foundational results.



\begin{proposition}
\label{prop:AHAT}
Let $T$ denote a single-layer GAHAT equipped with an FFN module that parameterizes the function $f$. For any arbitrary scalar $v \in \mathbb{R}$, there exists a corresponding single-layer GAHAT $T^\prime$ such that its constituent FFN module encodes the function $f^\prime(\cdot) = f(v \cdot)$.
\end{proposition}
\begin{proof}
    Let the attention mechanism of $T^\prime$ be parametrically identical to that of $T$. It remains to define the FFN component of $T^\prime$ such that it encodes the composed mapping $f^\prime(x) = f(vx)$. This requires the target FFN to perform a scalar multiplication by $v$ before the evaluation of the original function $f$. Followed by \cite{hornik1989multilayer}, a single FFN architecture is capable of representing this continuous compositional mapping to an arbitrary degree of precision, where $f$ is already implemented in the FFN of $T$. The existence of such an FFN configuration thereby completes the construction of $T^\prime$.
\end{proof}
Implication of the result (w.r.t RASP): The RASP primitive \lstinline{aggregate} inherently computes a mean over the attended value sequence. Consider the case where each row of the attention score matrix contains exactly $v$ dominant entries. Under this condition, the aggregation yields the sum scaled by a factor of $1/v$. By invoking the preceding proposition, the subsequent FFN component can be parameterized to apply a scalar multiplication of $v$. This effectively neutralizes the averaging mechanism to recover the summation of values.
\begin{proposition}
    \label{1D}
    Let $\mathcal{M}$ and $\mathcal{M}^\prime$ denote the spaces of $n$-dimensional and $1$-dimensional arrays, respectively. For any function $g : \mathcal{M} \to \mathcal{M}$, there exists a corresponding function $g^\prime : \mathcal{M}^\prime \to \mathcal{M}^\prime$ such that $g(A) = \prop^{-1}(g^\prime(\prop(A)))$, where $\prop : \mathcal{M} \to \mathcal{M}^\prime$ is a bijection.
\end{proposition}
\begin{proof}
    Let the construction $g^\prime$ be amenable to basic arithmetic operations. Note that to obtain $g$ using $g^\prime$ it is enough to show that $\prop$ is implementable using basic arithmetic operations. We would apply induction on $n$ to prove this and hence to show that all elements can be accessed uniquely. Avoiding the trivial case when $n=1$, let us consider a two-dimensional array $A$ having $m_0$ rows and $m_1$ columns. Note that any element $A[i_0,i_1]; i_l \in [m_l-1]_{\cup \{0\}}, l\in\{0,1\}$ can be accessed from the one-dimensional array $A^\prime$ of size $m_0m_1$ using the elementary index calculation $m_0i_0+i_1$. Let there exist a one-dimensional representation $A^\prime$ for the $n$-dimensional array $A$. To construct an one-dimensional representation $A^{\prime\prime}$ for an $n+1$-dimensional array $A$, let $A^{\prime}_{j}$ denotes the one-dimensional representation of the $n$-dimensional array $A[j], j \in[m_0-1]_{\cup \{0\}}$. Now, concatenation of these $A^{\prime}_{j}$, say $A^{\prime\prime}$ is the one-dimensional representation of $A$ where an element $A[i_0, i_1, \ldots, i_n]$ can be accessed in $A^{\prime\prime}$ using index $\prod_{l=1}^{n}m_li_0+k$, where $k$ is the index of the element in $A^{\prime}_{i_0}$.
\end{proof}

Since transformers require sequential inputs, the above bijection between matrices and sequences is crucial. Building upon this, the following lemmas demonstrate the existence of transformers that execute the target matrix operations. Note that here $\prop$ has been implemented using the widely known row-major format; a column-major format may also be preferred.
In combination with \autoref{prop:AHAT}, proving these existences, and thereby the constructions involve finding suitable attention matrices over a sequence of steps that, upon applying a $\circ$ to the (intermediate) input results in the target matrix operation.


\begin{lemma}
\label{lem:transposition}
    Given any matrix $A$ of order $r$, there exists a transformer $T_{\top}^{(r)}$ such that, $T_{\top}^{(r)}(\prop(A)) = \prop(A^\top)$.
\end{lemma} 
\begin{proof}
    For $A \in \R^{r\times c}$, where $c \in \mathbb{N}$; a transformer map that results in $\prop(A^\top)$, essentially carries out a permutation $\rho$ (say) on the elements of $\prop(A)$. As such, proving the existence of a transformer boils down to constructing an AHA-realizable matrix $R$ of attention scores that rearranges $\prop(A)$ under the operation $\circ$ according to permutation $\rho$. In \autoref{algo:trans} (Line 2), we present the exact construction of $R$ as a function of $\rho \in \Pi([rc-1]_{\cup \{0\}})$, such that $T_{\top}^{(r)}(\prop(A)) \coloneqq R \circ \prop(A)$. Note that the resultant $R$ turns out to be a stochastic matrix, and therefore, immediately AHA-realizable (by taking $v=1$ in \autoref{prop:AHAT}). 
\end{proof}

\noindent Note that for square matrices, the \autoref{algo:trans} does not require the order $r$ explicitly. Since implementations using RASP allow any arithmetic computation, $r$ can be determined from the expression $r^2 = \text{\lstinline{length}}$.  

\begin{lemma}
\label{lem:softmax}
    There exists a transformer $T_{\sigma}^{(r)}$ implementing the operation $\softmax$ on matrix $A$ of order $r$, i.e., $T_{\sigma}^{(r)}(\prop(A)) = \prop(\sigma(A))$.
\end{lemma}
\begin{proof}
    Similar to transposition (\autoref{lem:transposition}), the operation in hand is length-preserving on a matrix $A$. The proof also follows a similar pathway, given the sufficiency due to the existence of AHA-realizable attention matrices. However, the challenge lies in appropriately scaling for each nested array in $\prop(A)$ (rows of $A$) along the order $r$. Thus, we construct attention matrices corresponding to each index $l \in [r]$ (\autoref{algo:softmax}, Line 2), that conform to the operation $\circ$, and eventually, due to the AHA-realizability of the scaling using $/$ gives a feasible $T_{\sigma}^{(r)}$.
\end{proof}

\noindent Note that any RASP implementation of \autoref{algo:softmax} undergoes the obvious approximation of the expansion $\exp(\cdot)$ (Line 1). While the resultant transformer, due to RASP, can thus be said to simulate the activation approximately, the existence (\autoref{lem:softmax}) holds nonetheless.   


\begin{algorithm}
    \SetAlgoLined\SetAlgoNoLine\SetNlSty{}{}{}\LinesNumbered\RestyleAlgo{ruled}\DontPrintSemicolon\SetAlgoSkip{}\SetKwComment{tcp}{\textcolor{blue}{$\vartriangleright$ }}{}\SetCommentSty{texttt}
    \KwIn{$\prop(A)$, where $A \in \R^{r\times c}$.}
    Let $\rho \in \Pi([rc-1]_{\cup \{0\}})$ such that $\rho([rc-1]_{\cup \{0\}}) = \{rj+i\}_{j \in [c-1]_{\cup \{0\}} \text{ for each } i \in [r-1]_{\cup \{0\}}}$.\; \tcp{\small Choosing a permutation $\rho$ of indices of $A$ that maps element $A[i,j]$ to position $(j,i)$.}
    Assign $R[\rho([rc-1]_{\cup \{0\}})[ci+j], ci+j] = 1$ to $R \in \mathbf{0}^{rc \times rc}$, for $i$ and $j$ as in line 1.\; \tcp{\small Creates an attention $R$ that maps the indices of $A$ to the reflected indices from $\rho$.}
    \Return $R \circ \prop(A)$.\;  
    \caption{ Transposing a matrix $A$ of order $r$. \hspace{\fill} [see \autoref{lst:transpose}]}
    \label{algo:trans}
\end{algorithm}
\begin{algorithm}
    \SetAlgoLined\SetAlgoNoLine\SetNlSty{}{}{}\LinesNumbered\RestyleAlgo{ruled}\DontPrintSemicolon\SetAlgoSkip{}\SetKwComment{tcp}{\textcolor{blue}{$\vartriangleright$ }}{}\SetCommentSty{texttt}
    \KwIn{$\prop(A)$, where $A = [a_{ij}] \in \R^{r\times c}$.}
    $A^\prime = [a_{ij}^\prime]$ such that $a_{ij}^\prime = \exp(a_{ij}).$\;
    Let $R_l = \begin{pmatrix}
        & \mathbf{0}^{r(l-1)\times rc} & \\
        \mathbf{0}^{r\times (l-1)c} & \mathbf{1}^{r\times c} & \mathbf{0}^{r\times (r-l)c}\\
        & \mathbf{0}^{r(c-l)\times rc} & \\
    \end{pmatrix}$, where $l \in [r]$.\; \tcp{\small Creates $r$ attention matrices each drawing the length-preserved sequence $A^{\prime}[i]$ padded with $0$, where $i \in [r-1]_{\cup\{0\}}.$}
    Let $\operatorname{sum} \in \R^{rc}$ such that $\operatorname{sum} = \sum_l \left(R_l \circ \prop(A^\prime)\right)$.\; \tcp{\small Creates an array $\operatorname{sum}$ such that  $\operatorname{sum}[k]=\sum_{j}A^{\prime}[i,j]$ for all $ci\le k < c(i+1)$.}
    \Return $\prop(A^\prime) / \operatorname{sum}$.
    \caption{\small Applying $\softmax$ on a matrix $A$ of order $r$. \hspace{\fill} [see \autoref{lst:softmax}]}
    \label{algo:softmax}
\end{algorithm}
\begin{algorithm}
    \SetAlgoLined\SetAlgoNoLine\SetNlSty{}{}{}\LinesNumbered\RestyleAlgo{ruled}\DontPrintSemicolon\SetAlgoSkip{}\SetKwComment{tcp}{\textcolor{blue}{$\vartriangleright$ }}{}\SetCommentSty{texttt}
    \KwIn{$\prop(A)\concat\prop(B)$, where $A \in \R^{r\times k}$ and $B \in \R^{k \times c}$.} \tcp{\small Let $r \:(\text{and\:} c) \mapsfrom $ order of $A$ (and $B^\top$).}
    Assign $(R_l)[i,j] = 1$ for $i \in [k(r+c)-1]_{\cup \{0\}}$ and $j=k(l-1) + (i \mod k)$ to $R_l \in \mathbf{0}^{(k(r+c)) \times (k(r+c))}$, where $l \in [r]$.\; \tcp{\small Create $r$ attention matrices each drawing the length-preserved sequence $A[i, :]$, where $i \in \{0,1,\dots,r-1\}.$}
    Assign $(R^\prime_l)[i,j] = 1$ for $i \in [k(r+c)-1]_{\cup \{0\}}$ and $j= rk + (l-1) + c(i \mod k)$ to $R^\prime_l \in \mathbf{0}^{(k(r+c)) \times (k(r+c))}$, where $l \in [c]$.\; \tcp{\small  Similarly, create $c$ attention matrices each drawing the length-preserved sequence $B[:,j]$.}
    Let $M_{ll^\prime} = (R_l \circ (\prop(A)\concat\prop(B)))\otimes(R^\prime_{l^\prime} \circ (\prop(A)\concat\prop(B)))$ such that $l \in [r]$ and $l^\prime \in [c]$.\; \tcp{\small Multiply tokens from each row of $A$ with that of each column of $B$ and store the $rc$ sequences in $rc$ variables.}
    Assign $(R^{\prime\prime}_l)[i,j] = 1$ for $i = l-1$ and $j \in [k-1]_{\cup\{0\}}$ to $R^{\prime\prime}_l \in \mathbf{0}^{(k(r+c)) \times (k(r+c))}$, where $l \in [rc]$.\; \tcp{\small Create $rc$ attention matrices such that attention matrix $i$ focuses on first $\cdot$ positions of $i$\textsuperscript{th} row.}
    \Return $\sum_{l,l^\prime} R^{\prime\prime}_{\tilde{l}k+\tilde{l^\prime}+1} \circ  M_{ll^\prime}$, where $\tilde{l} = l-1$ and $\tilde{l^\prime} = l^\prime-1$.\;
    \tcp{\small Combine the sequences from line 4 with the attention matrices produced from line 5 to get $AB$, where the last $\cdot\:(r+c) - rc$ tokens are $0$.}
    \caption{\small Multiplication of matrices $A$ and $B$ of shape $r \times \cdot \:\text{ and } \cdot \times c$ respectively. \hspace{\fill} [see \autoref{lst:matmul}]}
    \label{algo:matmul}
\end{algorithm}

\begin{lemma}
\label{lem:matmul}
    There exists a transformer $T_{\boxtimes}^{(r,c)}$ multiplying matrices $A$ and $B$ of shape $r \times k$ and $k \times c$, for any $k \ge \frac{rc}{r+c}$.
\end{lemma}
\begin{proof}
    Based on \autoref{1D}, let us consider that the input matrices $A \in \R^{r\times k}$ and $B \in \R^{k \times c}$, mapped to sequences, are concatenated one after another. The proof begins with constructing a pair of attention matrices (see Lines 1 and 2, \autoref{algo:matmul}) that are tailored to draw out targeted rows and columns of $A$ and $B$. The next step follows a similar $\circ$ operation, this time along with an element-wise scalar product to result in $rc$ replicates of $A[i, :]^{\top}B[:, j]$, $i \in [r-1]_{\cup\{0\}}$, $j \in [c-1]_{\cup\{0\}}$.\footnote{Equivalent to $A[i]$ from \autoref{sec:prelim}; we use $A[i, :]$ following standard matrix conventions.} Finally, one requires an attention matrix to sequentially pick up the product output tokens ($A[i, :]^{\top}B[:, j]$) under the operation $\circ$. The most crucial feature of this step in the construction (Line 5) is the reparametrization of the attention coordinates. On aggregation, the resultant array has $\prop(AB)$ as its first $rc$ tokens, with the rest ($k(r+c)-rc$) padded with $0$. The padded output is a direct consequence of the operation $\concat$ inherently violating the length preservation property while implementing multiplication if $k> rc/(r+c)$.
\end{proof}

To address this issue of redundant tokens (\autoref{algo:matmul}) appearing at the end of the sequence, we may incorporate a trivial attention mechanism in conjunction with a feed-forward network. This approach enables the contraction of a sequence with $m$ tokens into a shorter sequence of length $n$ (where $n<m$). To achieve this, a weight matrix $W^{m\times n}$ is employed within the final feed-forward sublayer such that $W[i,j] = 1$ when $i \le n$, and $i = j$; and $0$ otherwise. Having implemented the fundamental operations with transformers, we now present our main result. 

\begin{theorem}
\label{Th:1}
    There exists a transformer network $\mathcal{U}$ that, on any input $X$ of shape $(n\times d)$, can simulate any single-layer transformer attention $T$ of order $(n,d,d_v)$.
\end{theorem} 
\begin{proof}
    Observe that the restricted transformer attention $T$ is characterized by $A$ and $V$ such that it can be expressed as $\sigma\left(X AX^\top\right)XV$. The network $\mathcal{U}$ simulating $T$ on input $X$ takes input $X, A \text{ and } V$; and it can be constructed through a series of fundamental operations, each of which can be implemented using \autoref{algo:trans}-\ref{algo:matmul} specifically by $T_{\top}^{(n)}, T_{\sigma}^{(n)}, T_{\boxtimes}^{(n, d)}, T_{\boxtimes}^{(n, n)}, \text{ and } T_{\boxtimes}^{(n, d_v)}$. In addition to the schematic diagram of the required network $\mathcal{U}$ as presented in \autoref{fig:U}, \autoref{fig:simulationexample} provides a demonstrative simulation. 
\end{proof}

Note that the criteria in \autoref{lem:matmul} can be satisfied by $k \ge \min(r,c)$. In the context of \autoref{Th:1}, this requires $n \le \min(d, d_v)$. This not only aligns with the existing empirical scenarios where the sequence length $n$ is smaller than the representation dimensionality $d$ \citep{vaswani2017attention}, but it also renders the relation between hidden dimensions immaterial.  Additionally, the representational dimensions ($d$ and $d_v$) are often considered equal. In such scenarios, given that $n \operatorname{mod} 4 = 0$, the construction of $\mathcal{U}$ becomes entirely dependent on the sequence length, i.e., the number of input symbols. The reason being, given the RASP primitive \lstinline{length}, the provision to perform any arithmetic operation, and the value of $n$, we may deduce the value $d$. For example, to know the value of $d$ while multiplying $X$ and $A$, we may evaluate the expression $nd + d^2 = \text{\lstinline{length}}$.

\begin{figure}[!ht]
    \centering
\resizebox{\textwidth}{!}{
\begin{tikzpicture}[
    >=Stealth,
    node distance=1.5cm and 0.5cm,
    inputnode/.style={
        rectangle,
        draw=none,       
        fill=white,       
        line width=0.5pt, 
        inner sep=6pt,
        align=center,
        font=\scriptsize
    },
    opnode/.style={
        rectangle,
        draw=black,
        fill=white,
        line width=0.2pt,
        inner xsep=-2.5pt,
        inner ysep=0pt,
        font=\scriptsize
    },
    link/.style={
        ->,
        draw=black!80,
        rounded corners=2pt 
    }
]

\newcommand{\OpContent}[5]{
    \setlength\arrayrulewidth{0.2pt}
    \begin{tabular}{
        @{} 
        >{\centering\arraybackslash}m{2cm} @{\hspace{1.5pt}\vline\hspace{1.5pt}}
        >{\centering\arraybackslash}m{2.1cm} @{\hspace{1.5pt}\vline\hspace{1.5pt}}
        >{\centering\arraybackslash}m{1.4cm} 
        @{}
    }
        \tiny\texttt{\textcolor{blue!40!black}{\ \ #1}} & 
        \tiny $|\textcolor{blue!40!black}{\text{\texttt{#4}}} - \prop(\textcolor{green!40!black}{\text{\texttt{#4}}})| #5$ & 
        \tiny\texttt{\textcolor{green!30!black}{#2}} \\
        \noalign{\hrule height 0.2pt}
        \multicolumn{3}{@{}c@{}}{
            \rule{0pt}{2ex} 
            \tiny $ \textcolor{blue!40!black}{\text{\texttt{#4}}} = #3 $
            \rule{0pt}{1.5ex} 
        }
    \end{tabular}
}


\node[opnode] (OpMix) {
    \OpContent{Y1 = MatMul() ($\prop$(X)$\concat\prop$(A))}{Y1 = X @ A}{
        0.51, 0.66, 0.81, 0.65, 0.79, 0.93, 0.67, 0.83, 0.99
    }{Y1}{\approx 0.0}
};

\node[opnode, left=0.5cm of OpMix] (OpTrans) {
    \OpContent{XT = Transpose() ($\prop$(X))}{XT = X.T}{
        0.8, 0.1, 0.6, 0.2, 0.9, 0.3, 0.5, 0.4, 0.7
    }{XT}{ = 0}
};

\node[opnode, right=0.5cm of OpMix] (OpVal) {
    \OpContent{Y4 = MatMul() ($\prop$(X)$\concat\prop$(V))}{Y4 = X @ V}{
        0.53, 0.72, 0.8, 1.11, 0.42, 1.09, 0.71, 0.65, 0.93
    }{Y4}{\approx 0.0}
};


\node[inputnode, above=1cm of OpTrans] (InputX) {
    \textbf{Input $X$}\\
    $\begin{bmatrix} 
    .8 & .2 & .5 \\ .1 & .9 & .4 \\ .6 & .3 & .7 
    \end{bmatrix}$
};

\node[inputnode, above=1cm of OpMix] (InputA) {
    \textbf{Input $A$}\\
    $\begin{bmatrix} 
    .1 & .2 & .3 \\ .4 & .5 & .6 \\ .7 & .8 & .9 
    \end{bmatrix}$
};

\node[inputnode, above=1cm of OpVal] (InputV) {
    \textbf{Input $V$}\\
    $\begin{bmatrix} 
    .1 & .7 & .4 \\ 1.0 & .3 & .9 \\ .5 & .2 & .6 
    \end{bmatrix}$
};


\node[opnode] (OpScore) at ($(OpTrans.south)!0.5!(OpMix.south) + (0,-1.1)$) {
    \OpContent{Y2 = MatMul()(Y1$\concat$XT)}{\hspace*{-8mm}Y2 = Y1 @ X.T}{
        0.945, 0.969, 1.071, 1.143, 1.148, 1.278, 1.197, 1.21, 1.344
    }{Y2}{\approx 0.0}
};

\node[opnode, below=0.4cm of OpScore] (OpSoft) {
    \OpContent{Y3 = Softmax()(Y2)}{\hspace*{-8mm}Y3 = softmax(S)}{
        0.317, 0.324, 0.359, 0.318, 0.319, 0.363, 0.315, 0.319, 0.365
    }{Y3}{\approx 3.376358e-06}
};

\node[opnode] (OpFinal) at ($(OpSoft.south)!0.5!(OpVal.south) + (0,-2.7)$) {
    \OpContent{Y = MatMul()(Y3$\concat$Y4)}{\hspace*{-8mm}Y = Y3 @ Y4}{
        0.783, 0.598, 0.941, 0.780, 0.599, 0.940, 0.781, 0.599, 0.940
    }{Y}{\approx 3.040329e-07}
};

\draw[link] (InputX.south) -- (OpTrans.north);
\draw[link] (InputX.south) -- ++(0,-0.4) [rounded corners=5pt] -| (OpMix.30);
\draw[link] (InputA.south) -- (OpMix.north);

\draw[link] (InputX.south) -- ++(0,-0.4) -| (OpVal.150);
\draw[link] (InputV.south) -- (OpVal.north);

\draw[link] (OpTrans.south) -- ++(0,-0.3) -| (OpScore.160);
\draw[link] (OpMix.south) -- ++(0,-0.3) -| (OpScore.30);

\draw[link] (OpScore.south) -- (OpSoft.north);

\draw[link] (OpSoft.south) -- ++(0,-0.5) -| (OpFinal.160);
\draw[link] (OpVal.south) -- ++(0,-0.5) -| (OpFinal.30);


\node[
    rectangle,
    draw=black!50,
    dashed,
    fill=white,
    rounded corners=2pt,
    inner sep=4pt,
    anchor=west,
    font=\tiny,
] (Legend) at ($(OpFinal.east) + (0.15, 1.9)$) {
    \color{black!60}
    \begin{tabular}{@{}p{8mm} p{2.5cm}@{}}
        \multicolumn{2}{@{}p{2.5cm}}{Our Implementation} \\
        \hline
        \rule{0pt}{2.2ex}\texttt{\textcolor{blue!40!black}{Transpose}} & $\to$ \autoref{lst:transpose} \\
        \texttt{\textcolor{blue!40!black}{Softmax}} & $\to$  \autoref{lst:softmax} \\
        \texttt{\textcolor{blue!40!black}{MatMul}} & $\to$  \autoref{lst:matmul} \\[0.8ex]
        \multicolumn{2}{@{}p{2.5cm}}{NumPy Implementation} \\
        \hline
        \texttt{\textcolor{green!40!black}{T}} & $\to$ Predefined Property \\
        \texttt{\textcolor{green!40!black}{Softmax}} & $\to$  Custom Function \\
        \texttt{\textcolor{green!40!black}{@}} & $\to$  Predefined Operator\\[0.8ex]
        \multicolumn{2}{@{}l}{Find at} \\
        \hline
        \multicolumn{2}{@{}p{3.7cm}}{\fontsize{5pt}{5pt}\selectfont \url{https://anonymous.4open.science/r/TMA/Demo/README.md}}
    \end{tabular}
};
\end{tikzpicture}
}
    \caption{A demonstrative simulation comparing the transformer network $\mathcal{U}$ (vide \autoref{Th:1}) against standard linear algebra operations via NumPy. Note that, the negligible error (maximum of absolute difference) appearing in the Softmax calculation (\texttt{\textcolor{blue!40!black}{Y3}}), along with the near-zero error in calculation of matrix multiplications \iffalse propagating to the final aggregation (\texttt{\textcolor{blue!40!black}{Y}}) \fi are attributed to the fixed-precision constraint (three significant digits) applied to the exponential function in the simulation and floating-point representation of numbers in NumPy respectively.}
    \label{fig:simulationexample}
\end{figure}

\begin{corollary}
\label{col:enc}
    There exists a transformer network $\mathcal{U}$ that can simulate any single-layer transformer encoder $T$ of order $(n, d, d_v, d_1, d_2)$.
\end{corollary} 
\begin{proof}
    The characterizing parameters of a transformer encoder from the class of restricted transformer contain two additional matrices. Let $W_1$ and $W_2$ specify the linear projections within the feed-forward sublayer, in addition to the attention and value matrices characterizing $T$'s self-attention sublayer. That is $T(X) = \operatorname{FFN}\left(X + \sigma\left(X AX^\top\right)XV\right)$ where, $\operatorname{FFN}(X) = \operatorname{ReLU}\left(X W_1\right)W_2$.\\
    Besides implementing the operation sum in the residual connection similar to line 3 of \autoref{algo:softmax}, we can realize the activation $\operatorname{ReLU}$ (see \autoref{lst:relu}). Now, continuing from \autoref{Th:1},  rest of the operations can be simulated using an application of \autoref{algo:matmul}.
\end{proof}

\begin{remark}[Simulating SMAT using AHAT]
\label{rem:ahattosmat}
    The theorem ensures the existence of a unified network capable of simulating certain computational models while maintaining parity with the models under simulation. Precisely, we have employed average hard attention (see GAHAT in \autoref{sec:prelim}) to mimic $\softmax$-activated attentions. As a consequence, we can utilize our AHAT network for problems including $\text{Match}_2$, known until now to be only learnable using single-layer single-head SMATs (see \autoref{sec:1}).
\end{remark}
\begin{remark}[Simulating Linformer and Linear Attention]
\label{rem:linformer}
    As long as the characterizing matrices of the transformers are involved with matrix multiplication (e.g., Linformer (\ref{eq:linformer})) and the function $\Pool$ is implementable using GAHAT or RASP (e.g., linear attention (\ref{eq:linatt})), the \autoref{Th:1} and \autoref{col:enc} can be applied to achieve a transformer network $\mathcal{U}$ simulating them.
\end{remark}

Remarkably, one may follow an alternative approach to proving the representational capacity of $\mathcal{U}$ by showing that it realizes operations such as (\ref{eq:inv}) (see Appendix \ref{ssec:inverse}). The proof involves altering the construction of $\mathcal{U}$ by introducing final attention parameters that adapt to the input $\langle T, X\rangle$. It is crucial since, in the process, we show the existence of a transformer that inverts non-singular matrices of fixed orders. See Appendix \ref{ssec:inverse} for a contextual discussion. 

\section{Discussion on Generalization}
\label{sec:discussion}
\begin{wraptable}[11]{r}{0.4\textwidth}
    \centering
    \vspace{-1\intextsep}
    \caption{The computational cost of construction associated with the operations and whether they are dependent on the order of the input matrices.}
    \resizebox{0.4\textwidth}{!}{
    \begin{tabular}{l|c|c}
    \toprule
        Operation & Input Dependency & Cost\\ \midrule
        Transposition & \cmark & $O(1)$\\
        $\operatorname{softmax}$ & \cmark & $O(r)$\\
        Multiplication & \cmark & $O(rc)$\\
        $\operatorname{MaxMin}$ & \xmark & $O(1)$\\
    \bottomrule
    \end{tabular}}
    \label{tab:1}
\end{wraptable}

Let us first analyze the complexity of the constructions given above. We define the \emph{width} of a single encoder layer as the count of attention heads it contains. To extend this definition to multi-layer encoders, we define the width as the maximum width among all its constituent single-layer encoders. The shortcoming that makes the \autoref{algo:matmul} lengthy stems from explicitly mentioning the $r+c+rc$ attention matrices. Even with classical implementation of matrix multiplication, where $C[i_0,i_1] = \sum_{i=1}^k A[i_0,i]B[i,i_1]$, taking $O(rkc)$ time, it does not resolve the issue, but rather follows the same in the scope of variable renaming facility. In contrast, since our constructed transformer $T_{\boxtimes}^{(r,c)}$ assumes attention being one of the basic operations and thus is an $O(1)$ operation, matrix multiplication costs $O(2rc)$ number of operations. Similarly, the computation cost for \autoref{algo:softmax} and \autoref{algo:trans} for an order-$r$ matrix is $O(r)$ and $O(1)$, respectively. For each algorithm, the construction of the transformers ensures that their depth is not a function of the input; however, for most cases, the width \emph{is} {--} a comprehensive view has been presented in \autoref{tab:1}.

\begin{corollary}
\label{col:expressive_hierarchy}
Let $\mathcal{U}_{(n,d,d_v)}$ and $\mathcal{U}_{(m,e,e_v)}$ be transformer networks as defined in \autoref{Th:1}, where 
    i) $\mathcal{U}_{(n,d,d_v)}$ simulates single-layer transformers with characterizing matrices $A \in \mathbb{R}^{d \times d}$ and $V \in \mathbb{R}^{d \times d_v}$, and inputs $X \in \mathbb{R}^{n \times d}$,
    ii) $\mathcal{U}_{(m,e,e_v)}$ is defined analogously for dimensions $m$, $e$, and $e_v$.
If $n \geq m$, $d \geq e$, and $d_v \geq e_v$, then $\mathcal{U}_{(n,d,d_v)}$ is at least as expressive as $\mathcal{U}_{(m,e,e_v)}$. Specifically, any computation performed by $\mathcal{U}_{(m,e,e_v)}$ can be exactly simulated by $\mathcal{U}_{(n,d,d_v)}$.
\end{corollary}

The corollary signifies the notion of hierarchy in simulation power. Our construction of a suitable $\mathcal{U}$, as discussed after \autoref{Th:1}, ensures the existence of a computational model that can simulate any single-layer transformer attention with a given number of heads based on input $X$. We highlight that it is $X$ that dictates the width of $\mathcal{U}$, whose depth remains independent of the input. As such, the construction hinges solely on the sequence length $n$. For a sufficiently large $N$, we can inductively construct and hence prove the existence of a network, say $\mathcal{U}_{(N)}$ that can simulate arbitrary attention (or, even transformer when extended with the feed-forward component) on input having length, say $n \le N$ {--} thus making it universal. The constructive proof has been provided in Appendix \ref{ssec:proofofhierarchy}. In the absence of a theoretical lower bound on the allowable number of heads, we can only ensure that the dependence underlying our model follows the principle of parsimony, given the natural hierarchy among simulators.

\begin{remark}[Sparsification]
    Note that the $\operatorname{poly}(n)$ complexity underlying our construction of $\mathcal{U}$ stems from the definition of vanilla encoders, and does not contribute to inflation of ambient sequence length. \autoref{Th:1} only requires $n \leq \min(d, d_v)$, which conforms to the convention in \citet{vaswani2017attention}. Remarkably, our approach also conforms to sparsification of pairwise token interactions, namely, methods that involve pooling to achieve appropriate compression and low-rank attentions, e.g., Linformer \citep{wang2020linformer}, Performer \citep{choromanski2021rethinking}, and Sumformer \citep{alberti2023sumformer}. This becomes crucial in mitigating the commonly encountered issue of token explosion. Linformer approximates the self-attention mechanism by a low-rank matrix \textemdash the lower rank ($k$) being prescribed based on the Johnson-Lindenstrauss Lemma \textemdash to achieve a complexity of $O(n)$. Meanwhile, Performer replaces the usual non-linearity by introducing kernels for pooling. Under Gaussian kernels, the complexity can be made as low as $O(nkd)$. Sumformers consolidate all of the above models in universally approximating sequence-to-sequence permutation-equivariant continuous functions. Our construction can be used to represent all such models as long as the underlying pooling operations are representable (see \autoref{rem:linformer}).
\end{remark}

In the purview of \autoref{lem:softmax}, we also extend the encoder's expressivity onto a larger class of activations. First, suppose $S$ is a sequence of length $gk$ and $\rho_g$ is a permutation, where $g,k$ are positive integers. Thus, $\rho_g(S)$ is the $g$-sorted sequence of $S$ such that $\rho_g(S)[i]\ge\rho_g(S)[i+1]\ge\cdots\ge\rho_g(S)[i+g-1]$ for all $i$ that are multiple of $g$. This permutation is often called a $\operatorname{GroupSort}$ of group size $g$. When $g=2$, this is widely known as the $\operatorname{MaxMin}$ operation \citep{anil2019sorting}. 

\begin{lemma}
    \label{lem:MaxMin}
    There exists a transformer $T_{\rho_{2}}$ implementing the operation $\operatorname{MaxMin}$ on a sequence $S$ (of even length), i.e., $T_{\rho_{2}}(S) = \rho_{2}(S)$.
\end{lemma}
\begin{proof}
    In the context of our discussion, sequences may be understood as $\prop(A)$, where $A$ is an input matrix of arbitrary order. Moreover, similar to most activation implementations, $\operatorname{MaxMin}$ is length-preserving. The key observation in constructing a feasible attention matrix (\autoref{algo:maxmin}, Line 2), that upon $\circ$ yields the sorted sequence, is the encoding of tokens according to their mutual order. Note that the construction does not readily extend to arbitrary grouping sizes, and for our proof, it is essential to have $|S|=2p$, $p \in \mathbb{N}$ (even order in case of $\prop(\cdot)$). Remarkably, since the number of attentions does not depend on the input, the AHA-realization is only $O(1)$-costly. 
\end{proof}

\begin{remark}[Approximating Lipschitz functions]
    The first reason behind \autoref{lem:MaxMin} being important is that, by representing $\operatorname{MaxMin}$, $\mathcal{U}$ can express a vector $p$-norm preserving transform, $p \geq 1$. As such, recalling that $\mathcal{U}$ also simulates affine matrix operations (multiplication), it can represent an $L$-deep feed-forward network $z^{(\ell)} \coloneqq W^{(\ell)} \operatorname{MaxMin}(z^{(\ell-1)}) + b^{(\ell)}$, where $W^{(\ell)} \in \mathbb{R}^{n_{\ell} \times n_{\ell-1}}, b^{(\ell)} \in \mathbb{R}^{n_{\ell}}$, given that ${||W^{1}||}_{2, \infty} \leq 1$, $\max\{{||W^{(\ell)}||}_{\infty}\}_{\ell=2}^{L} \leq 1$ and $\max\{{||b^{(\ell)}||}_{\infty}\}_{\ell=1}^{L} \leq \infty$. In case the input vectors $z^{0}$ are constrained to a compact subset $Z \subseteq \mathbb{R}^{n_{0}}$ and $n_{L}=1$, the simulated outputs are dense in ${Lip}_{1}(Z)$ \citep{tanielian2021approximating}. This presents a new proof showing that transformer encoders are universal approximators of Lipschitz and H\"older-smooth functions. Moreover, following \autoref{lem:MaxMin}, $\mathcal{U}$ exactly represents $\operatorname{ReLU}, \operatorname{Leaky ReLU}$ and $\operatorname{Maxout}$ activations \citep{anil2019sorting}. The result extends to $\operatorname{GeLU}$-activated networks given the approximation of $\operatorname{GeLU}$ (Lipschitz-smooth with associated constant 1.0998) using $\operatorname{ReLU}$ (\citet{feng2023towards}, Lemma C.2).
\end{remark}

\begin{algorithm}
    \SetAlgoLined\SetAlgoNoLine\SetNlSty{}{}{}\LinesNumbered\RestyleAlgo{ruled}\DontPrintSemicolon\SetKwComment{tcp}{\textcolor{blue}{$\vartriangleright$ }}{}\SetCommentSty{texttt}
    \KwIn{Sequence $S = \{S[i]\}_{i\in [|S|-1]_{\cup \{0\}}}$, where $|S| = 2p, p\in \N$.}
    Assign $\alpha[i] = 1$ (and $\alpha[i+1] = 1$) if $S[i]<S[i+1]$ to $\alpha \in \{0\}^{|S|}$, where $i \in [|S|-2]_{\cup \{0\}}$ and $i\mod 2 = 0$.\; \tcp{\small Indicates the positions of the pairs of elements that are ordered (ascending).} 
    Assign $R[i,i] = 1-\alpha[i] = R[i+1, i+1]$ and  $R[i, i+1] = \alpha[i] = R[i+1, i]$ to $R \in \mathbf{0}^{|S|\times |S|}$, where $i\mod 2 = 0$. \;    \tcp{\small Creates an attention matrix $R$ with diagonal blocks $\mathbf{1}_2 - \mathbf{I}_2$ (or, $\mathbf{I}_2$) depending on whether $\alpha[i] \text{ and } \alpha[i+1]$ are both $1 (\text{or, }0)$ for any even $i$.}
    \Return $R \circ S$. \hfill {\textcolor{blue}{$\vartriangleright$} \texttt{\small Return the tokens of $S$ after passing through $R$.}}\;
    \caption{\small Applying $\operatorname{MaxMin}$ sort on any sequence $S$. \hspace{\fill} [see \autoref{lst:MaxMin}]}
    \label{algo:maxmin}
\end{algorithm}

To further generalize the construction of $\mathcal{U}$, let us now work on the multi-head extension.

\begin{lemma}
    \label{lem:multihead}
    Suppose $\f$ is an $n$-ary operation. Then, there exists a transformer $T_{h\f}$ realizing $\concat\limits_{h=1}^{H}\Big(\f\limits_{i=1}^{n}X_i^{(h)}\Big)$ on input $\concat\limits_{i=1}^{n}\Big(\concat\limits_{h=1}^{H}X_i^{(h)}\Big)$, if there is a transformer $T_{\f}$ realizing the operation $\f$ on input $\concat\limits_{i=1}^{n}X_i$, where $\concat$ denotes concatenation.
\end{lemma} 
\begin{proof}
    If the construction for operation $\f$ is independent of the input, $T_{h\f}=T_{\f}$ for any $h$, e.g., $\operatorname{MaxMin}$. Otherwise, we provide an explicit construction of such a transformer $T_{h\f}$. A transformer can implement the following operations:
    \begin{itemize}[leftmargin=*,itemsep=0pt, topsep=0pt]
        \item \textsf{identify}: A contiguous subsequence $s^\prime = \{s_i\}_{0 \le i^\prime-1 \le i < i^\prime+k \le |s|-1}$ from a sequence $s = \{s_i\}_{i\in [|s|-1]_{\cup \{0\}}}$ can be identified to produce a length-preserved sequence $0\dots0s_{i-1} \dots s_{i+k-1}0\dots0$. Assign $R[i, j] = 1$ for $0 \le i^\prime-1 \le i < i^\prime+k \le |s|-1$ and $j =i$ to $R \in \mathbf{0}^{|s|\times|s|}$. Then, $R \circ s$ performs \textsf{identify}. The corresponding RASP code is
        \lstinline{clip = select(indices, indices, ==)} \lstinline{and} \lstinline{select(indices, i-1, >=) and select(indices, i+k-1, <=);}
        \lstinline{aggregate(clip, tokens);}. Line 3 (and 2 \& 3) of \autoref{algo:softmax} (and \ref{algo:matmul}) reminisce the property.
        \item \textsf{shift}: A cyclic permutation $\rho_t$ can be performed, such that, $\rho_t(\{s_i\}_{i\in [|s|-1]_{\cup \{0\}}}) =  \{s_{(i+t)\mod |s|}\}_{i\in [|s|-1]_{\cup \{0\}}}$ on the sequence $s=\{s_i\}_{i\in [|s|-1]_{\cup \{0\}}}$. Assign $R[i, j] = 1$ for $i\in [|s|-1]_{\cup \{0\}}$ and $j = (|s| - t + i ) \mod |s|$ to $R \in \mathbf{0}^{|s|\times|s|}$. Then, $R \circ s$ performs \textsf{shift}. The corresponding RASP code is \lstinline{aggregate(select(indices, (indices+t)
    \end{itemize}
    Now to construct $T_{h\f}$, let us first apply the \textsf{identify} and \textsf{shift} operation to permute the input sequence $\concat\limits_{i=1}^{n}\Big(\concat\limits_{h=1}^{H}X_i^{(h)}\Big)$ to $\concat\limits_{h=1}^{H}\Big(\concat\limits_{i=1}^{n}X_i^{(h)}\Big)$. The operations of $T_{h\f}$ would then copy the operations from $T_{\f}$ with possible modification in \lstinline{indices}, which is in $T_{h\f}$, is at an offset $n\times (h^\prime-1)$. When applied on the sequence $\concat\limits_{h=1}^{h^\prime-1}\Big(\concat\limits_{i=1}^{n}0\Big)\concat\limits_{i=1}^{n}X_i^{(h^\prime)}\concat\limits_{h=h^\prime+1}^{H}\Big(\concat\limits_{i=1}^{n}0\Big)$, this would produce $\concat\limits_{h=1}^{h^\prime-1}\Big(\concat\limits_{i=1}^{n}0\Big)\f\limits_{i=1}^{n}X_i^{(h^\prime)}\concat\limits_{h=h^\prime+1}^{H}\Big(\concat\limits_{i=1}^{n}0\Big)$. Then just adding up all such $H$ sequences would produce $\concat\limits_{h=1}^{H}\Big(\f\limits_{i=1}^{n}X_i^{(h)}\Big)$. Note that if a transformer $T_{\f}$ requires using the FFN (e.g., the matrix multiplication), $T_{h\f}$ can also construct weight matrices for the FFN with possible modifications to cater only to the required portions of the produced sequence.
\end{proof}
The purpose of this lemma is to prove that given an operation $\f$ implementable by a transformer, another transformer can be constructed that can independently perform $\f$, say, $H$ times, without mutual interference. 

\begin{theorem}
    \label{thm:U_multihead}
    There exists a transformer network $\mathcal{U}$ that, on any input $X$, can simulate any single-layer $H$-head transformer attention $T$ of order $(n, d, d_v)$. 
\end{theorem} 
\begin{proof}
    Keeping congruence to the input provided to multihead attentions by  \citeauthor{vaswani2017attention} (Subsection 3.2.2), we assume that the characterizing matrices have been stacked one after another, i.e., $\Big(\concat\limits_{h=1}^{H}X^{(h)}\Big)\concat\Big(\concat\limits_{h=1}^{H}A^{(h)}\Big)\concat\Big(\concat\limits_{h=1}^{H}V^{(h)}\Big)$, where, $X^{(h)}\in\R^{n\times d}, A^{(h)}\in\R^{d\times d}$ and $V^{(h)}\in\R^{d\times d_v}$ is the input to network $\mathcal{U}$. Thus, the construction of $\mathcal{U}$ follows from \autoref{lem:multihead} and \autoref{lem:transposition}-\ref{lem:matmul}.
\end{proof}
Evidently, the result also extends to the entire transformer.
\begin{corollary}
    \label{col:U_multihead}
    There exists a transformer network $\mathcal{U}$ that, on any input $X$, can simulate any single-layer $H$-head transformer $T$ of order $(n, d, d_v, d_1, d_2)$. 
\end{corollary}
\begin{remark}
\label{rem:multihead}
    With the simulation of multi-head transformers, an architecture can be realized through explicit construction for the problems which are known to be learnable using single-layer multi-head transformers, e.g., $k$-\textsf{PARITY}. Note that in terms of RASP, a residual connection is only an elementwise sum. Accordingly, for the task of recognizing the language \textsf{PARITY}, the two-layer $\operatorname{softmax}$ encoder architecture proposed by \citeauthor{chiang2022overcoming} can be realized using average hard attention by employing two serially connected $\mathcal{U}$ networks.
\end{remark}

\section{Conclusion}
\label{sec:conclusion}
We present for the first time an exact, data-agnostic construction of a universal simulator that replicates the behavior of single-layer transformer encoders, including multi-head attention and non-linear feed-forward components. Central to our construction is the implementation of key linear algebraic operations and a wide-range of activation functions, all within the constant-depth constraint of transformer architecture. Our results demonstrate that while such structure precludes simulation of arbitrary encoder configurations, a hierarchical construction exists wherein simulators of higher-order subsume lower-order models. Crucially, extending this to multi-head attention as in \autoref{thm:U_multihead} ensures the existence of a universal simulator $\mathcal{U}$. As an obvious extension of this work and backed by the Turing completeness of transformers, one may investigate the construction of an analogous mechanism involving an encoder-decoder-based model to simulate an arbitrary transformer. 
By constructing average-hard attention-based models that exactly replicate $\operatorname{softmax}$-activated attention, we show that algorithmic approximations of problems previously believed to be learnable only through training, such as $\text{Match}_2$ and $k$-\textsf{PARITY}. This rigorously shifts the boundary between empirical approximation and formal simulation in attention-based models.
The development of RASP compilers such as Tracr \citep{lindner2023tracr} and ALTA \citep{shaw2025alta} presents a promising avenue for obtaining realized weights corresponding to the proposed network $\mathcal{U}$. However, challenges towards a complete implementation of the RASP framework still exist in either of these compilers. Along such a line, learning the algorithmically developed transformers and $\mathcal{U}$ via existing optimization-based methods and coming up with a convergence criterion may be considered as future work.





\bibliographystyle{plainnat}
\bibliography{ref}

\begin{thebibliography}{57}
\providecommand{\natexlab}[1]{#1}
\providecommand{\url}[1]{\texttt{#1}}
\expandafter\ifx\csname urlstyle\endcsname\relax
  \providecommand{\doi}[1]{doi: #1}\else
  \providecommand{\doi}{doi: \begingroup \urlstyle{rm}\Url}\fi

\bibitem[Ahn et~al.(2023)Ahn, Cheng, Daneshmand, and Sra]{ahn2023transformers}
Kwangjun Ahn, Xiang Cheng, Hadi Daneshmand, and Suvrit Sra.
\newblock {Transformers Learn to Implement Preconditioned Gradient Descent for In-Context Learning}.
\newblock In A.~Oh, T.~Naumann, A.~Globerson, K.~Saenko, M.~Hardt, and S.~Levine, editors, \emph{Advances in Neural Information Processing Systems}, volume~36, pages 45614--45650. Curran Associates, Inc., 2023.
\newblock URL \url{https://proceedings.neurips.cc/paper_files/paper/2023/file/8ed3d610ea4b68e7afb30ea7d01422c6-Paper-Conference.pdf}.

\bibitem[Alberti et~al.(2023)Alberti, Dern, Thesing, and Kutyniok]{alberti2023sumformer}
Silas Alberti, Niclas Dern, Laura Thesing, and Gitta Kutyniok.
\newblock {Sumformer: Universal Approximation for Efficient Transformers}.
\newblock In \emph{Proceedings of 2nd Annual Workshop on Topology, Algebra, and Geometry in Machine Learning (TAG-ML)}, volume 221 of \emph{Proceedings of Machine Learning Research}, pages 72--86. PMLR, 28 Jul 2023.
\newblock URL \url{https://proceedings.mlr.press/v221/alberti23a.html}.

\bibitem[Anil et~al.(2019)Anil, Lucas, and Grosse]{anil2019sorting}
Cem Anil, James Lucas, and Roger Grosse.
\newblock {Sorting Out Lipschitz Function Approximation}.
\newblock In \emph{International Conference on Machine Learning}, pages 291--301. PMLR, 2019.

\bibitem[Barcelo et~al.(2024)Barcelo, Kozachinskiy, Lin, and Podolskii]{barcelo2024logical}
Pablo Barcelo, Alexander Kozachinskiy, Anthony~Widjaja Lin, and Vladimir Podolskii.
\newblock {Logical Languages Accepted by Transformer Encoders with Hard Attention}.
\newblock In \emph{The Twelfth International Conference on Learning Representations}, 2024.
\newblock URL \url{https://openreview.net/forum?id=gbrHZq07mq}.

\bibitem[Bhattamishra et~al.(2020{\natexlab{a}})Bhattamishra, Ahuja, and Goyal]{bhattamishra2020ability}
Satwik Bhattamishra, Kabir Ahuja, and Navin Goyal.
\newblock On the {A}bility and {L}imitations of {T}ransformers to {R}ecognize {F}ormal {L}anguages.
\newblock In Bonnie Webber, Trevor Cohn, Yulan He, and Yang Liu, editors, \emph{Proceedings of the 2020 Conference on Empirical Methods in Natural Language Processing (EMNLP)}, pages 7096--7116, Online, November 2020{\natexlab{a}}. Association for Computational Linguistics.
\newblock \doi{10.18653/v1/2020.emnlp-main.576}.
\newblock URL \url{https://aclanthology.org/2020.emnlp-main.576/}.

\bibitem[Bhattamishra et~al.(2020{\natexlab{b}})Bhattamishra, Patel, and Goyal]{bhat2020computational}
Satwik Bhattamishra, Arkil Patel, and Navin Goyal.
\newblock {On the Computational Power of Transformers and Its Implications in Sequence Modeling}.
\newblock In Raquel Fern{\'a}ndez and Tal Linzen, editors, \emph{Proceedings of the 24th Conference on Computational Natural Language Learning}, pages 455--475, Online, November 2020{\natexlab{b}}. Association for Computational Linguistics.
\newblock \doi{10.18653/v1/2020.conll-1.37}.
\newblock URL \url{https://aclanthology.org/2020.conll-1.37/}.

\bibitem[Cheng et~al.(2024)Cheng, Chen, and Sra]{cheng2024transformers}
Xiang Cheng, Yuxin Chen, and Suvrit Sra.
\newblock {Transformers Implement Functional Gradient Descent to Learn Non-Linear Functions In Context}.
\newblock In Ruslan Salakhutdinov, Zico Kolter, Katherine Heller, Adrian Weller, Nuria Oliver, Jonathan Scarlett, and Felix Berkenkamp, editors, \emph{Proceedings of the 41st International Conference on Machine Learning}, volume 235 of \emph{Proceedings of Machine Learning Research}, pages 8002--8037. PMLR, 21--27 Jul 2024.
\newblock URL \url{https://proceedings.mlr.press/v235/cheng24a.html}.

\bibitem[Chiang(2025)]{chiang:2025}
David Chiang.
\newblock {Transformers in Uniform {TC$^0$}}.
\newblock \emph{Transactions on Machine Learning Research}, January 2025.
\newblock URL \url{https://openreview.net/forum?id=ZA7D4nQuQF}.

\bibitem[Chiang and Cholak(2022)]{chiang2022overcoming}
David Chiang and Peter Cholak.
\newblock {Overcoming a Theoretical Limitation of Self-Attention}.
\newblock In Smaranda Muresan, Preslav Nakov, and Aline Villavicencio, editors, \emph{Proceedings of the 60th Annual Meeting of the Association for Computational Linguistics (Volume 1: Long Papers)}, pages 7654--7664, Dublin, Ireland, May 2022. Association for Computational Linguistics.
\newblock \doi{10.18653/v1/2022.acl-long.527}.
\newblock URL \url{https://aclanthology.org/2022.acl-long.527/}.

\bibitem[Choromanski et~al.(2021)Choromanski, Likhosherstov, Dohan, Song, Gane, Sarlos, Hawkins, Davis, Mohiuddin, Kaiser, Belanger, Colwell, and Weller]{choromanski2021rethinking}
Krzysztof~Marcin Choromanski, Valerii Likhosherstov, David Dohan, Xingyou Song, Andreea Gane, Tamas Sarlos, Peter Hawkins, Jared~Quincy Davis, Afroz Mohiuddin, Lukasz Kaiser, David~Benjamin Belanger, Lucy~J Colwell, and Adrian Weller.
\newblock {Rethinking Attention with Performers}.
\newblock In \emph{International Conference on Learning Representations}, 2021.
\newblock URL \url{https://openreview.net/forum?id=Ua6zuk0WRH}.

\bibitem[Consens et~al.(2025)Consens, Dufault, Wainberg, Forster, Karimzadeh, Goodarzi, Theis, Moses, and Wang]{consens2025transformers}
Micaela~E Consens, Cameron Dufault, Michael Wainberg, Duncan Forster, Mehran Karimzadeh, Hani Goodarzi, Fabian~J Theis, Alan Moses, and Bo~Wang.
\newblock Transformers and genome language models.
\newblock \emph{Nature Machine Intelligence}, pages 1--17, 2025.

\bibitem[Dehghani et~al.(2019)Dehghani, Gouws, Vinyals, Uszkoreit, and Kaiser]{dehghani2019universal}
Mostafa Dehghani, Stephan Gouws, Oriol Vinyals, Jakob Uszkoreit, and Lukasz Kaiser.
\newblock {Universal Transformers}.
\newblock In \emph{International Conference on Learning Representations}, 2019.
\newblock URL \url{https://openreview.net/pdf?id=HyzdRiR9Y7}.

\bibitem[Deletang et~al.(2023)Deletang, Ruoss, Grau-Moya, Genewein, Wenliang, Catt, Cundy, Hutter, Legg, Veness, and Ortega]{deletang2023neural}
Gregoire Deletang, Anian Ruoss, Jordi Grau-Moya, Tim Genewein, Li~Kevin Wenliang, Elliot Catt, Chris Cundy, Marcus Hutter, Shane Legg, Joel Veness, and Pedro~A Ortega.
\newblock {Neural Networks and the Chomsky Hierarchy}.
\newblock In \emph{The Eleventh International Conference on Learning Representations}, 2023.
\newblock URL \url{https://openreview.net/forum?id=WbxHAzkeQcn}.

\bibitem[Edelman et~al.(2022)Edelman, Goel, Kakade, and Zhang]{edelman2022inductive}
Benjamin~L Edelman, Surbhi Goel, Sham Kakade, and Cyril Zhang.
\newblock {Inductive Biases and Variable Creation in Self-Attention Mechanisms}.
\newblock In \emph{International Conference on Machine Learning}, pages 5793--5831. PMLR, 2022.

\bibitem[Feng et~al.(2023)Feng, Zhang, Gu, Ye, He, and Wang]{feng2023towards}
Guhao Feng, Bohang Zhang, Yuntian Gu, Haotian Ye, Di~He, and Liwei Wang.
\newblock {Towards Revealing the Mystery behind Chain of Thought: A Theoretical Perspective}.
\newblock In \emph{Thirty-seventh Conference on Neural Information Processing Systems}, 2023.
\newblock URL \url{https://openreview.net/forum?id=qHrADgAdYu}.

\bibitem[Fu et~al.(2024)Fu, qi~Chen, Jia, and Sharan]{fu2024transformers}
Deqing Fu, Tian qi~Chen, Robin Jia, and Vatsal Sharan.
\newblock {Transformers Learn to Achieve Second-Order Convergence Rates for In-Context Linear Regression}.
\newblock In \emph{The Thirty-eighth Annual Conference on Neural Information Processing Systems}, 2024.
\newblock URL \url{https://openreview.net/forum?id=L8h6cozcbn}.

\bibitem[Giannou et~al.(2023)Giannou, Rajput, Sohn, Lee, Lee, and Papailiopoulos]{giannou2023looped}
Angeliki Giannou, Shashank Rajput, Jy-Yong Sohn, Kangwook Lee, Jason~D. Lee, and Dimitris Papailiopoulos.
\newblock {Looped Transformers as Programmable Computers}.
\newblock In Andreas Krause, Emma Brunskill, Kyunghyun Cho, Barbara Engelhardt, Sivan Sabato, and Jonathan Scarlett, editors, \emph{Proceedings of the 40th International Conference on Machine Learning}, volume 202 of \emph{Proceedings of Machine Learning Research}, pages 11398--11442. PMLR, 23--29 Jul 2023.
\newblock URL \url{https://proceedings.mlr.press/v202/giannou23a.html}.

\bibitem[Giannou et~al.(2025)Giannou, Yang, Wang, Papailiopoulos, and Lee]{giannou2024transformersimulate}
Angeliki Giannou, Liu Yang, Tianhao Wang, Dimitris Papailiopoulos, and Jason~D. Lee.
\newblock {How Well Can Transformers Emulate In-Context Newton's Method?}
\newblock In \emph{The 28th International Conference on Artificial Intelligence and Statistics}, 2025.
\newblock URL \url{https://openreview.net/forum?id=cj5L29VWol}.

\bibitem[Hahn(2020)]{hahn-2020-theoretical}
Michael Hahn.
\newblock {Theoretical Limitations of Self-Attention in Neural Sequence Models}.
\newblock \emph{Transactions of the Association for Computational Linguistics}, 8:\penalty0 156--171, 2020.
\newblock \doi{10.1162/tacl_a_00306}.
\newblock URL \url{https://aclanthology.org/2020.tacl-1.11/}.

\bibitem[Han and Ghoshdastidar(2025)]{han2025attention}
Yaomengxi Han and Debarghya Ghoshdastidar.
\newblock {Attention Learning is Needed to Efficiently Learn Parity Function}.
\newblock \emph{arXiv preprint arXiv:2502.07553}, 2025.

\bibitem[Hao et~al.(2022)Hao, Angluin, and Frank]{hao2022formal}
Yiding Hao, Dana Angluin, and Robert Frank.
\newblock {Formal language recognition by hard attention transformers: Perspectives from circuit complexity}.
\newblock \emph{Transactions of the Association for Computational Linguistics}, 10:\penalty0 800--810, 2022.
\newblock \doi{10.1162/tacl_a_00490}.
\newblock URL \url{https://aclanthology.org/2022.tacl-1.46/}.

\bibitem[Haruna et~al.(2025)Haruna, Qin, {Adama Chukkol}, Yusuf, Bello, and Lawan]{haruna2025exploring}
Yunusa Haruna, Shiyin Qin, Abdulrahman~Hamman {Adama Chukkol}, Abdulganiyu~Abdu Yusuf, Isah Bello, and Adamu Lawan.
\newblock {Exploring the synergies of hybrid convolutional neural network and Vision Transformer architectures for computer vision: A survey}.
\newblock \emph{Engineering Applications of Artificial Intelligence}, 144:\penalty0 110057, 2025.
\newblock ISSN 0952-1976.
\newblock \doi{https://doi.org/10.1016/j.engappai.2025.110057}.
\newblock URL \url{https://www.sciencedirect.com/science/article/pii/S0952197625000570}.

\bibitem[Hornik et~al.(1989)Hornik, Stinchcombe, and White]{hornik1989multilayer}
Kurt Hornik, Maxwell Stinchcombe, and Halbert White.
\newblock {Multilayer Feedforward Networks are Universal Approximators}.
\newblock \emph{Neural Networks}, 2\penalty0 (5):\penalty0 359--366, 1989.
\newblock ISSN 0893-6080.
\newblock \doi{https://doi.org/10.1016/0893-6080(89)90020-8}.
\newblock URL \url{https://www.sciencedirect.com/science/article/pii/0893608089900208}.

\bibitem[Katharopoulos et~al.(2020)Katharopoulos, Vyas, Pappas, and Fleuret]{LinearTransformer}
Angelos Katharopoulos, Apoorv Vyas, Nikolaos Pappas, and Fran\c{c}ois Fleuret.
\newblock {Transformers are RNNs: Fast Autoregressive Transformers with Linear Attention}.
\newblock In \emph{Proceedings of the 37th International Conference on Machine Learning}, ICML'20. JMLR.org, 2020.

\bibitem[Kim et~al.(2024)Kim, Nakamaki, and Suzuki]{kim2024transformers}
Juno Kim, Tai Nakamaki, and Taiji Suzuki.
\newblock {Transformers are Minimax Optimal Nonparametric In-Context Learners}.
\newblock In \emph{ICML 2024 Workshop on In-Context Learning}, 2024.
\newblock URL \url{https://openreview.net/forum?id=WjrKBQTWKp}.

\bibitem[Kudlek(2012)]{kudlekexistence}
Manfred Kudlek.
\newblock {On the Existence of Universal Finite or Pushdown Automata}.
\newblock \emph{Electronic Proceedings in Theoretical Computer Science}, July 2012.

\bibitem[Lin et~al.(2022)Lin, Wang, Liu, and Qiu]{lin2022survey}
Tianyang Lin, Yuxin Wang, Xiangyang Liu, and Xipeng Qiu.
\newblock A survey of transformers.
\newblock \emph{AI Open}, 3:\penalty0 111--132, 2022.
\newblock ISSN 2666-6510.
\newblock \doi{https://doi.org/10.1016/j.aiopen.2022.10.001}.
\newblock URL \url{https://www.sciencedirect.com/science/article/pii/S2666651022000146}.

\bibitem[Lindner et~al.(2023)Lindner, Kramar, Farquhar, Rahtz, McGrath, and Mikulik]{lindner2023tracr}
David Lindner, Janos Kramar, Sebastian Farquhar, Matthew Rahtz, Thomas McGrath, and Vladimir Mikulik.
\newblock {Tracr: Compiled Transformers as a Laboratory for Interpretability}.
\newblock In \emph{Thirty-seventh Conference on Neural Information Processing Systems}, 2023.
\newblock URL \url{https://openreview.net/forum?id=tbbId8u7nP}.

\bibitem[Liu et~al.(2023)Liu, Ash, Goel, Krishnamurthy, and Zhang]{liu2023transformers}
Bingbin Liu, Jordan~T. Ash, Surbhi Goel, Akshay Krishnamurthy, and Cyril Zhang.
\newblock {Transformers Learn Shortcuts to Automata}.
\newblock In \emph{The Eleventh International Conference on Learning Representations}, 2023.
\newblock URL \url{https://openreview.net/forum?id=De4FYqjFueZ}.

\bibitem[Luo et~al.(2022)Luo, Li, Zheng, Liu, Wang, and He]{luo2022your}
Shengjie Luo, Shanda Li, Shuxin Zheng, Tie-Yan Liu, Liwei Wang, and Di~He.
\newblock {Your Transformer May Not be as Powerful as You Expect}.
\newblock \emph{Advances in Neural Information Processing Systems}, 35:\penalty0 4301--4315, 2022.

\bibitem[Merrill and Sabharwal(2023)]{merrill2023logic}
William Merrill and Ashish Sabharwal.
\newblock {A logic for Expressing Log-Precision Transformers}.
\newblock In \emph{Proceedings of the 37th International Conference on Neural Information Processing Systems}, NIPS '23, Red Hook, NY, USA, 2023. Curran Associates Inc.
\newblock URL \url{https://dl.acm.org/doi/10.5555/3666122.3668406}.

\bibitem[Merrill and Sabharwal(2024)]{merrill2024the}
William Merrill and Ashish Sabharwal.
\newblock {The Expressive Power of Transformers with Chain of Thought}.
\newblock In \emph{The Twelfth International Conference on Learning Representations}, 2024.
\newblock URL \url{https://openreview.net/forum?id=NjNGlPh8Wh}.

\bibitem[Merrill et~al.(2020)Merrill, Weiss, Goldberg, Schwartz, Smith, and Yahav]{merrill-etal-2020-formal}
William Merrill, Gail Weiss, Yoav Goldberg, Roy Schwartz, Noah~A. Smith, and Eran Yahav.
\newblock {A Formal Hierarchy of {RNN} Architectures}.
\newblock In Dan Jurafsky, Joyce Chai, Natalie Schluter, and Joel Tetreault, editors, \emph{Proceedings of the 58th Annual Meeting of the Association for Computational Linguistics}, pages 443--459, Online, July 2020. Association for Computational Linguistics.
\newblock \doi{10.18653/v1/2020.acl-main.43}.
\newblock URL \url{https://aclanthology.org/2020.acl-main.43/}.

\bibitem[Merrill et~al.(2022)Merrill, Sabharwal, and Smith]{merrill2022saturated}
William Merrill, Ashish Sabharwal, and Noah~A Smith.
\newblock {Saturated Transformers are Constant-Depth Threshold Circuits}.
\newblock \emph{Transactions of the Association for Computational Linguistics}, 10:\penalty0 843--856, 2022.
\newblock \doi{10.1162/tacl_a_00493}.
\newblock URL \url{https://aclanthology.org/2022.tacl-1.49/}.

\bibitem[Mroueh(2023)]{mroueh2023towards}
Youssef Mroueh.
\newblock {Towards a Statistical Theory of Learning to Learn In-context with Transformers}.
\newblock In \emph{NeurIPS 2023 Workshop Optimal Transport and Machine Learning}, 2023.
\newblock URL \url{https://openreview.net/forum?id=ZbioTIO6y6}.

\bibitem[Pathak et~al.(2024)Pathak, Sen, Kong, and Das]{pathak2024transformers}
Reese Pathak, Rajat Sen, Weihao Kong, and Abhimanyu Das.
\newblock Transformers can optimally learn regression mixture models.
\newblock In \emph{The Twelfth International Conference on Learning Representations}, 2024.
\newblock URL \url{https://openreview.net/forum?id=sLkj91HIZU}.

\bibitem[P{\'e}rez et~al.(2019)P{\'e}rez, Marinkovi{\'c}, and Barcel{\'o}]{perez2019turing}
Jorge P{\'e}rez, Javier Marinkovi{\'c}, and Pablo Barcel{\'o}.
\newblock {On the Turing Completeness of Modern Neural Network Architectures}.
\newblock \emph{arXiv preprint arXiv:1901.03429}, 2019.

\bibitem[P{\'e}rez et~al.(2021)P{\'e}rez, Barcel{\'o}, and Marinkovic]{perez2021attention}
Jorge P{\'e}rez, Pablo Barcel{\'o}, and Javier Marinkovic.
\newblock {Attention is Turing-Complete}.
\newblock \emph{Journal of Machine Learning Research}, 22\penalty0 (75):\penalty0 1--35, 2021.
\newblock URL \url{http://jmlr.org/papers/v22/20-302.html}.

\bibitem[Qiu et~al.(2025)Qiu, Xu, Bao, and Tong]{qiu2025ask}
Ruizhong Qiu, Zhe Xu, Wenxuan Bao, and Hanghang Tong.
\newblock {Ask, and it shall be given: On the Turing completeness of prompting}.
\newblock In \emph{The Thirteenth International Conference on Learning Representations}, 2025.
\newblock URL \url{https://openreview.net/forum?id=AS8SPTyBgw}.

\bibitem[Sanford et~al.(2023)Sanford, Hsu, and Telgarsky]{sanford2023representational}
Clayton Sanford, Daniel Hsu, and Matus Telgarsky.
\newblock Representational strengths and limitations of transformers.
\newblock In \emph{Thirty-seventh Conference on Neural Information Processing Systems}, 2023.
\newblock URL \url{https://openreview.net/forum?id=36DxONZ9bA}.

\bibitem[Shaw et~al.(2025)Shaw, Cohan, Eisenstein, Lee, Berant, and Toutanova]{shaw2025alta}
Peter Shaw, James Cohan, Jacob Eisenstein, Kenton Lee, Jonathan Berant, and Kristina Toutanova.
\newblock {{ALTA}: Compiler-Based Analysis of Transformers}.
\newblock \emph{Transactions on Machine Learning Research}, 2025.
\newblock ISSN 2835-8856.
\newblock URL \url{https://openreview.net/forum?id=h751wl9xiR}.

\bibitem[Shi et~al.(2022)Shi, Gao, Tian, Chen, and Zhao]{Shi_Gao_Tian_Chen_Zhao_2022}
Hui Shi, Sicun Gao, Yuandong Tian, Xinyun Chen, and Jishen Zhao.
\newblock {Learning Bounded Context-Free-Grammar via LSTM and the Transformer: Difference and the Explanations}.
\newblock \emph{Proceedings of the AAAI Conference on Artificial Intelligence}, 36\penalty0 (8):\penalty0 8267--8276, Jun. 2022.
\newblock \doi{10.1609/aaai.v36i8.20801}.
\newblock URL \url{https://ojs.aaai.org/index.php/AAAI/article/view/20801}.

\bibitem[Siegelmann and Sontag(1992)]{Hava1992}
Hava~T Siegelmann and Eduardo~D Sontag.
\newblock {On the Computational Power of Neural Nets}.
\newblock In \emph{Proceedings of the fifth annual workshop on Computational learning theory}, pages 440--449, 1992.
\newblock URL \url{https://doi.org/10.1145/130385.130432}.

\bibitem[Strobl et~al.(2024{\natexlab{a}})Strobl, Angluin, Chiang, Rawski, and Sabharwal]{strobl2024transformers}
Lena Strobl, Dana Angluin, David Chiang, Jonathan Rawski, and Ashish Sabharwal.
\newblock {Transformers as Transducers}.
\newblock \emph{Transactions of the Association for Computational Linguistics}, 13:\penalty0 200--219, 02 2024{\natexlab{a}}.
\newblock ISSN 2307-387X.
\newblock \doi{10.1162/tacl_a_00736}.
\newblock URL \url{https://doi.org/10.1162/tacl\_a\_00736}.

\bibitem[Strobl et~al.(2024{\natexlab{b}})Strobl, Merrill, Weiss, Chiang, and Angluin]{strobl2024formal}
Lena Strobl, William Merrill, Gail Weiss, David Chiang, and Dana Angluin.
\newblock {What Formal Languages Can Transformers Express? A Survey}.
\newblock \emph{{Transactions of the Association for Computational Linguistics}}, 12:\penalty0 543--561, 2024{\natexlab{b}}.
\newblock \doi{10.1162/tacl_a_00663}.
\newblock URL \url{https://aclanthology.org/2024.tacl-1.30/}.

\bibitem[Tanielian and Biau(2021)]{tanielian2021approximating}
Ugo Tanielian and Gerard Biau.
\newblock Approximating lipschitz continuous functions with groupsort neural networks.
\newblock In \emph{International Conference on Artificial Intelligence and Statistics}, pages 442--450. PMLR, 2021.

\bibitem[Vaswani et~al.(2017)Vaswani, Shazeer, Parmar, Uszkoreit, Jones, Gomez, Kaiser, and Polosukhin]{vaswani2017attention}
Ashish Vaswani, Noam Shazeer, Niki Parmar, Jakob Uszkoreit, Llion Jones, Aidan~N Gomez, \L~ukasz Kaiser, and Illia Polosukhin.
\newblock {Attention is All you Need}.
\newblock In I.~Guyon, U.~Von Luxburg, S.~Bengio, H.~Wallach, R.~Fergus, S.~Vishwanathan, and R.~Garnett, editors, \emph{Advances in Neural Information Processing Systems}, volume~30. Curran Associates, Inc., 2017.
\newblock URL \url{https://proceedings.neurips.cc/paper_files/paper/2017/file/3f5ee243547dee91fbd053c1c4a845aa-Paper.pdf}.

\bibitem[Wang et~al.(2020)Wang, Li, Khabsa, Fang, and Ma]{wang2020linformer}
Sinong Wang, Belinda~Z Li, Madian Khabsa, Han Fang, and Hao Ma.
\newblock {Linformer: Self-attention with Linear Complexity}.
\newblock \emph{arXiv preprint arXiv:2006.04768}, 2020.

\bibitem[Wei et~al.(2022)Wei, Chen, and Ma]{wei2021advances}
Colin Wei, Yining Chen, and Tengyu Ma.
\newblock {Statistically Meaningful Approximation: a Case Study on Approximating Turing Machines with Transformers}.
\newblock In S.~Koyejo, S.~Mohamed, A.~Agarwal, D.~Belgrave, K.~Cho, and A.~Oh, editors, \emph{Advances in Neural Information Processing Systems}, volume~35, pages 12071--12083. Curran Associates, Inc., 2022.
\newblock URL \url{https://proceedings.neurips.cc/paper_files/paper/2022/file/4ebf1d74f53ece08512a23309d58df89-Paper-Conference.pdf}.

\bibitem[Weiss et~al.(2021)Weiss, Goldberg, and Yahav]{weiss21a}
Gail Weiss, Yoav Goldberg, and Eran Yahav.
\newblock {Thinking Like Transformers}.
\newblock In Marina Meila and Tong Zhang, editors, \emph{Proceedings of the 38th International Conference on Machine Learning}, volume 139 of \emph{Proceedings of Machine Learning Research}, pages 11080--11090. PMLR, 18--24 Jul 2021.
\newblock URL \url{https://proceedings.mlr.press/v139/weiss21a.html}.

\bibitem[Yang and Chiang(2024)]{yang2024counting}
Andy Yang and David Chiang.
\newblock {Counting Like Transformers: Compiling Temporal Counting Logic Into Softmax Transformers}.
\newblock In \emph{First Conference on Language Modeling}, 2024.
\newblock URL \url{https://openreview.net/forum?id=FmhPg4UJ9K}.

\bibitem[Yang et~al.(2024{\natexlab{a}})Yang, Chiang, and Angluin]{yang2024masked}
Andy Yang, David Chiang, and Dana Angluin.
\newblock {Masked Hard-Attention Transformers Recognize Exactly the Star-Free Languages}.
\newblock In \emph{The Thirty-eighth Annual Conference on Neural Information Processing Systems}, 2024{\natexlab{a}}.
\newblock URL \url{https://openreview.net/forum?id=FBMsBdH0yz}.

\bibitem[Yang et~al.(2024{\natexlab{b}})Yang, Strobl, Chiang, and Angluin]{yang2024simulating}
Andy Yang, Lena Strobl, David Chiang, and Dana Angluin.
\newblock {Simulating Hard Attention Using Soft Attention}.
\newblock \emph{arXiv preprint arXiv:2412.09925}, 2024{\natexlab{b}}.

\bibitem[Yau et~al.(2024)Yau, Aky{\"u}rek, Mao, Tenenbaum, Jegelka, and Andreas]{yau2024learning}
Morris Yau, Ekin Aky{\"u}rek, Jiayuan Mao, Joshua~B Tenenbaum, Stefanie Jegelka, and Jacob Andreas.
\newblock {Learning Linear Attention in Polynomial Time}.
\newblock \emph{arXiv preprint arXiv:2410.10101}, 2024.

\bibitem[Yun et~al.(2020)Yun, Bhojanapalli, Rawat, Reddi, and Kumar]{Yun2020Are}
Chulhee Yun, Srinadh Bhojanapalli, Ankit~Singh Rawat, Sashank Reddi, and Sanjiv Kumar.
\newblock {Are Transformers universal approximators of sequence-to-sequence functions?}
\newblock In \emph{International Conference on Learning Representations}, 2020.
\newblock URL \url{https://openreview.net/forum?id=ByxRM0Ntvr}.

\bibitem[Zhang et~al.(2024)Zhang, Frei, and Bartlett]{zhang2024trained}
Ruiqi Zhang, Spencer Frei, and Peter~L. Bartlett.
\newblock {Trained Transformers Learn Linear Models In-Context}.
\newblock \emph{Journal of Machine Learning Research}, 25\penalty0 (49):\penalty0 1--55, 2024.
\newblock URL \url{http://jmlr.org/papers/v25/23-1042.html}.

\bibitem[Zhou et~al.(2024)Zhou, Bradley, Littwin, Razin, Saremi, Susskind, Bengio, and Nakkiran]{zhou2024what}
Hattie Zhou, Arwen Bradley, Etai Littwin, Noam Razin, Omid Saremi, Joshua~M. Susskind, Samy Bengio, and Preetum Nakkiran.
\newblock {What Algorithms can Transformers Learn? A Study in Length Generalization}.
\newblock In \emph{The Twelfth International Conference on Learning Representations}, 2024.
\newblock URL \url{https://openreview.net/forum?id=AssIuHnmHX}.

\end{thebibliography}

\appendix
\newpage
\section{Appendix}
\label{sec:appendix}

\subsection{RASP Codes for Implementing the network \texorpdfstring{$\mathcal{U}$}{}}
\label{ssec:RASPCodes}
This section will provide the RASP codes referenced in \autoref{sec:simulation}.
\begin{lstlisting}[caption={Transposing a matrix of order $3$ implementing \autoref{algo:trans}. Note that the transpose operation is a length-preserving operation.}, breakindent=2.5mm,label={lst:transpose}]
  def Transpose_r(){
    r, c = 3, length/3;
    reflectedIndices = (indices%r)*c + ((indices-indices%r))/r;
    reflect = select(indices, reflectedIndices, ==);
    return aggregate(reflect, tokens);
  }
\end{lstlisting}
\begin{lstlisting}[caption={Applying $\softmax$ (a length-preserving operation) on matrix $A$ of order $3$ implementing \autoref{algo:softmax}.}, breakindent=2.5mm,label={lst:softmax}]
  def softmaxrect_r(){
    r, c = 3, length/3;
    exp = (2.73^tokens_float);
    sel1, sel2, sel3 =(select(indices, c*0+c, <) and select(c*0+c, indices, >)), (select(indices, c*0+c, >=) and select(indices, c*1+c, <) and select(c*1+c, indices, >) and select(c*0+c, indices, <=)), (select(indices, c*1+c, >=) and select(indices, c*2+c, <) and select(c*2+c, indices, >) and select(c*1+c, indices, <=));
    denom1, denom2, denom3 = c*aggregate(sel1, exp), c*aggregate(sel2, exp), c*aggregate(sel3, exp);
    denom = (denom1+denom2+denom3);
    return exp/denom;
  }
\end{lstlisting}
\begin{lstlisting}[caption={Multiplying two matrices of shape $3\times\cdot \text{ and } \cdot\times 4$ implementing \autoref{algo:matmul}.}, breakindent=2.5mm, label={lst:matmul}]
  def Matmul_3dot4(){
    k = length/(3+4);

    one_a, one_b, two_a, two_b, three_a, three_b, four_b = indices%k, (indices%k)*4+3*k, (indices%k)+1*k, (indices%k)*4+3*k+1, (indices%k)+2*k, (indices%k)*4+3*k+2, (indices%k)*4+3*k+3;

    one_sa, one_sb, two_sa, two_sb, three_sa, three_sb, four_sb = select(indices, one_a, ==), select(indices, one_b, ==), select(indices, two_a, ==), select(indices, two_b, ==), select(indices, three_a, ==), select(indices, three_b, ==), select(indices, four_b, ==);

    oneone_ab, onetwo_ab, onethree_ab, onefour_ab, twoone_ab, twotwo_ab, twothree_ab, twofour_ab, threeone_ab, threetwo_ab, threethree_ab, threefour_ab = aggregate(one_sa, tokens)*aggregate(one_sb, tokens), aggregate(one_sa, tokens)*aggregate(two_sb, tokens), aggregate(one_sa, tokens)*aggregate(three_sb, tokens), aggregate(one_sa, tokens)*aggregate(four_sb, tokens), aggregate(two_sa, tokens)*aggregate(one_sb, tokens), aggregate(two_sa, tokens)*aggregate(two_sb, tokens), aggregate(two_sa, tokens)*aggregate(three_sb, tokens), aggregate(two_sa, tokens)*aggregate(four_sb, tokens), aggregate(three_sa, tokens)*aggregate(one_sb, tokens), aggregate(three_sa, tokens)*aggregate(two_sb, tokens), aggregate(three_sa, tokens)*aggregate(three_sb, tokens), aggregate(three_sa, tokens)*aggregate(four_sb, tokens);

    sel_one, sel_two, sel_three, sel_four, sel_five, sel_six, sel_seven, sel_eight, sel_nine, sel_ten, sel_eleven, sel_twelve = select(indices, k, <) and select(0, indices, ==), select(indices, k, <) and select(1, indices, ==), select(indices, k, <) and select(2, indices, ==), select(indices, k, <) and select(3, indices, ==), select(indices, k, <) and select(4, indices, ==), select(indices, k, <) and select(5, indices, ==), select(indices, k, <) and select(6, indices, ==), select(indices, k, <) and select(7, indices, ==), select(indices, k, <) and select(8, indices, ==), select(indices, k, <) and select(9, indices, ==), select(indices, k, <) and select(10, indices, ==), select(indices, k, <) and select(11, indices, ==);

    matmul = k*(aggregate(sel_one, oneone_ab)+aggregate(sel_two, onetwo_ab)+aggregate(sel_three, onethree_ab)+aggregate(sel_four, onefour_ab)+aggregate(sel_five, twoone_ab)+aggregate(sel_six, twotwo_ab)+aggregate(sel_seven, twothree_ab)+aggregate(sel_eight, twofour_ab)+aggregate(sel_nine, threeone_ab)+aggregate(sel_ten, threetwo_ab)+aggregate(sel_eleven, threethree_ab)+aggregate(sel_twelve, threefour_ab));
      return matmul;
  }
\end{lstlisting}
\begin{lstlisting}[caption={Implementation of $\operatorname{ReLU}$.}, breakindent=3.5mm, label={lst:relu}]
  def ReLU(){
    return (0 if tokens<0 else tokens);
  }
\end{lstlisting}

\begin{lstlisting}[caption={Implementation of $\operatorname{MaxMin}$ sort realizing \autoref{algo:maxmin}.}, breakindent=2.5mm, label={lst:MaxMin}]
  def MaxMinSort(){
    MaxSel = select(indices, indices, ==);
    MinSel = select(indices, indices+1, ==) and select(1, indices%2+1, ==);  
    MaxminusMin = aggregate(MaxSel, tokens) - aggregate(MinSel, tokens);
    reqFlip = 1 if MaxminusMin<0 else 0;
    reqFlip = reqFlip + aggregate(select(indices+1, indices, ==), reqFlip);
    revby2 = aggregate(select(0, indices%2, ==), 1) + aggregate(select(1, indices%2, ==), -1); 
    flip = select(indices, indices+revby2, ==);
    sorted = reqFlip*aggregate(flip, tokens) + (1-reqFlip)*aggregate(select(indices, indices, ==), tokens);
    return sorted;
  }
\end{lstlisting}

\begin{lstlisting}[caption={Implementation of $\text{Match}_2$.}, breakindent=2.5mm, label={lst:Match2}]
  def Match2(p){
    sel = select(tokens%p, p - (tokens%p), ==);
    detect = aggregate(sel, tokens);
    return (0 if detect == 0 else 1);
  }
\end{lstlisting}
\paragraph{Complexity of individual operations.} To illustrate the complexities in accordance with the discussion in \autoref{sec:discussion}, we present the following analysis on the RASP codes. \autoref{lst:transpose} generates an attention matrix that maps the indices to their transposed position and then passes the tokens to get permuted accordingly. This requires a single non-trivial attention layer and a preceding layer to compute \lstinline{reflectedIndices}. The \lstinline{reflectedIndices} is in fact the permutation $\rho$ as defined in line 1 of \autoref{algo:trans}. A preceding layer with a trivial attention computes auxiliary arithmetic operations as specified in line 2 and 3 of \autoref{lst:transpose} via its FFN layer. For implementing the operation $\softmax$ in \autoref{lst:softmax}, the principal attention layer has a width of $3$, computing the row sums. Note that, the matrices $R_l$s in line 2 of \autoref{algo:softmax} is being realized by this layer. However, a preceding layer contains a trivial attention followed FFN layers  that computes the arithmetic expressions in line 4 (e.g. \lstinline{c*0+c}) of \autoref{lst:softmax}. Similarly, the function matrix multiplication in \autoref{lst:matmul} requires two non-trivial layers and an opening layer for the calculation of several index manipulations. Line 3 is performing necessary index calculations for implementing lines 2, 3 of the respective algorithm. The second attention layer corresponds to lines 4-5 and thus has a width of seven, while the third layer, corresponding to lines 6-7, has width twelve. For easy understanding, we have presented the keyword \lstinline{tokens} in the above code blocks; however, following the RASP semantics, we have used \lstinline{tokens_float} (or \lstinline{tokens_int}) while dealing with numerals in our code repository. A standard construction of UTM stores the transitions of an input TM using some delimiter (mostly a predefined number of $0$s). One may get intimidated to apply the same to delimit the rows of a matrix when presented as a sequence using \autoref{1D}. Though that would help to count the rows and thus columns, the architecture of transformers inhibits us from directly looping on the rows or columns to bypass the explicit construction of the \lstinline{select}-\lstinline{aggregate} pairs (e.g., the three selectors \lstinline{sel1}, \lstinline{sel2}, and \lstinline{sel3} in \autoref{lst:softmax}).

\subsection{Proof of \autoref{col:expressive_hierarchy}}
\label{ssec:proofofhierarchy}
    We prove this by construction. Let $T$ be a single-layer transformer with characterizing matrices $A \in \mathbb{R}^{e \times e}$ and $V \in \mathbb{R}^{e \times e_v}$, and input $X \in \mathbb{R}^{m \times e}$. To simulate $T$ using $\mathcal{U}_{(n,d,d_v)}$, we proceed as follows:
    
    \noindent \emph{Step 1} (Input Embedding).
    Pad the input $X$ to $\widetilde{X} \in \mathbb{R}^{n \times d}$ via zero-padding and block-diagonal extension:
    \[
    \widetilde{X} = \begin{bmatrix}
    X & \mathbf{0}^{m \times (d - e)} \\
    \mathbf{0}^{(n - m) \times e} & \mathbf{0}^{(n - m) \times (d - e)}
    \end{bmatrix},
    \]
    where $\mathbf{0}^{p \times q}$ denotes a zero matrix of size $p \times q$.
    
    \noindent \emph{Step 2} (Embedding the Characterizing Matrices).
    Similarly, embed $A$ and $V$ into higher-dimensional spaces:
    \[
    \widetilde{A} = \begin{bmatrix}
    A & \mathbf{0} \\
    \mathbf{0} & \mathbf{I}^{d - e}
    \end{bmatrix}, \quad
    \widetilde{V} = \begin{bmatrix}
    V & \mathbf{0} \\
    \mathbf{0} & \mathbf{0}
    \end{bmatrix},
    \]
    where $\mathbf{I}^{k}$ is the $k \times k$ identity matrix. The identity block ensures that padded dimensions do not interfere with the computation.
    
    \noindent \emph{Step 3} (Simulation).
    By construction, $\mathcal{U}_{(n,d,d_v)}$ computes:
    \[
    \sigma\left(\frac{\widetilde{X} \widetilde{A} \widetilde{X}^\top}{\sqrt{d}}\right) \widetilde{X} \widetilde{V},
    \]
    which reduces to the original computation $T(X)$ in the upper-left $m \times e_v$ block. The padded dimensions contribute only trivial linear transformations (due to $\mathbf{0}$ and $\mathbf{I}$ blocks), leaving the simulation exact.

\subsection{Alternative Construction for \texorpdfstring{$\mathcal{U}$}{}}
\label{ssec:inverse}
Before the discussion for constructing $\mathcal{U}$, let us implement another elementary matrix operation {--} inversion.

\begin{lemma}
\label{inversion}
    There exists a transformer $T_{-1}$ that can invert a non-singular matrix $A$ of rank $3$.
\end{lemma}
\begin{proof}
    Given a matrix $A$ and its mapped sequence from \autoref{1D}, the inversion operation is also length-preserving. We shall adopt an analytical framework for inverse computation, utilizing the fundamental operations of matrix cofactor, determinant, and adjugate. The final transposition step can be derived through the application of \autoref{lem:transposition}. The RASP pseudocode is presented in \autoref{algo:matinverse}.
    \begin{algorithm}
        \small
        \SetAlgoLined\SetAlgoNoLine\SetNlSty{}{}{}\LinesNumbered\RestyleAlgo{ruled}\DontPrintSemicolon\SetKwComment{tcp}{\textcolor{blue}{$\vartriangleright$ }}{}\SetCommentSty{texttt}
        \KwIn{$\prop(A)$, where $A \in \R^{r\times r}$.}
        \tcp{\scriptsize $r \mapsfrom$ rank of $A$.}
        Assign $R_\alpha[j, i] = 1$, for all $i \in [r^2-1]_{\cup\{0\}}$ and $j = r\times(((i-i \mod r)/r+1)\mod r) +( ((i \mod r) + 1) \mod r)$, to $R_\alpha \in \mathbf{0}^{r^2\times r^2}$.\;  
        Assign $R_\beta[j, i] = 1$, for all $i \in [r^2-1]_{\cup\{0\}}$ and $j = r\times(((i-i\mod r)/r + 2) \mod r) + (((i\mod r)+2)\mod r)$, to $R_\beta \in \mathbf{0}^{r^2\times r^2}$.\; 
        Assign $R_\gamma[j, i] = 1$, for all $i \in [r^2-1]_{\cup\{0\}}$ and $j = r\times (((i - i\mod r)/n + 1)\mod r) + (((i\mod r)+2)\mod r)$, to $R_\gamma \in \mathbf{0}^{r^2\times r^2}$.\; 
        Assign $R_\delta[j, i] = 1$, for all $i \in [r^2-1]_{\cup\{0\}}$ and $j = r\times (((i-i\mod r)/r+2)\mod r) + (((i\mod r)+1)\mod r)$, to $R_\delta \in \mathbf{0}^{r^2\times r^2}$.\; 
        Let $P = R_\alpha \circ \prop(A), Q = R_\beta \circ \prop(A), R = R_\gamma \circ \prop(A)$ and $S = R_\delta \circ \prop(A)$.\;
        $M_A = P\otimes Q - R\otimes S$.\; \tcp{\scriptsize $M_A$ is the cofactor of matrix $A$.}
        Assign $R_{\text{mask}}[i,j] = 1$ for $i = j$ and $j < r$, to $R_{\text{mask}} \in \mathbf{0}^{r^2\times r^2}$. \; \tcp{\scriptsize Create a mask that attends to the first $r$ positions of $\prop(A) \otimes M_A$.}
        $\operatorname{det}(A) = \mathbf{1}^{r^2\times r^2} \circ \text{ Masked}$, where $\text{Masked} = R_{\text{Mask}} \circ (\prop(A) \otimes M_A)$.\;
        return $M_A$ and $\operatorname{det}(A)$.\; 
        \caption{\small Finding Cofactor and Determinant of square matrix $A$ of rank $3$.}
        \label{algo:matinverse}
    \end{algorithm}
    
    The RASP code for finding cofactor and determinant has been provided in \autoref{lst:cofactor}, \ref{lst:det}.
    
    The expression $M_A/|A|$ (or, \lstinline{Cofactor()/Det(Cofactor())(A)}) now yields the transpose of the adjugate $\operatorname{Adj}(A)$. Subsequently, applying the transposition operation to the adjugate results in the desired matrix inverse $A^{-1}$.
\end{proof}

\begin{lstlisting}[caption={Finding Cofactor of a Matrix as a part of implementing \autoref{algo:matinverse}.}, breakindent=3.5mm, label={lst:cofactor}]
  def Cofactor(){
    n = length^0.5;
    i,j = (indices-indices%n)/n, indices%n;
    
    idx1, idx2, idx3, idx4 = (i+1)%n, (j+1)%n, (i+2)%n, (j+2)%n;
    one, two, three, four = idx3*n+idx4, idx1*n+idx2, idx3*n+idx2, idx1*n+idx4;
      
    sel_one, sel_two, sel_three, sel_four = select(indices, one, ==), select(indices, two, ==),select(indices, three, ==), select(indices, four, ==);
    P, Q, R, S = aggregate(sel_one, tokens), aggregate(sel_two, tokens), aggregate(sel_three, tokens), aggregate(sel_four, tokens);
    cofactor = P*Q-R*S;
    return cofactor;
  }
\end{lstlisting}
\begin{lstlisting}[caption={Finding Determinant of a Matrix as a part of implementing \autoref{algo:matinverse}.}, breakindent=3.5mm, label={lst:det}]
    def Det(Cofactor){
        n = length^0.5;
        mask = select(indices, n, <) and select(indices, indices, ==);
        det = length*aggregate(full_s, aggregate(mask, (tokens*Cofactor)));
        return det;
    }
\end{lstlisting}

Consider a transformer network $\mathcal{U}$ that adapts its parameters in the final attention layer depending on the input parameters $A, V$, and $X$, where $A \text{ and } V$ are the non-singular characterizing matrices of an attention $T$. Although such a construction does not solely satisfy the motivation as depicted in \autoref{fig:U}, it may be worth an attempt to explore the expressive power of a transformer implemented using RASP.

From the fact that $\det(M_1M_2) = \det(M_1)\det(M_2) \ne 0$ when matrices $M_1$ and $M_2$ are both non-singular, it is straightforward to see that matrix $M_1M_2$ is also non-singular. Now, we will construct $\mathcal{U}$, which will take input a sequence $X, A$ and $V$ such that the final attention layer, say $L$ receives input $XAV$. Thus, to simulate $T$ on input $X$, the attention and value matrices at layer $L$ of $\mathcal{U}$ must be $\left((AV)^\top V\right)^{-1}$ and $\left(AV\right)^{-1}V$ respectively, so that it produces:
\begin{align}
    \sigma\left(\left(XAV\right)\left((AV)^\top V\right)^{-1}\left(XAV\right)^\top\right)\left(XAV\right)\left(A{V}\right)^{-1}V \label{eq:inv}
\end{align}

As long as the characterizing matrices of $T$ are non-singular and of rank $3$, the attention and value matrices of $\mathcal{U}$ can be realized through a series of sequential operations implementable using the previous lemmas.

\begin{figure}
    \centering
    \includegraphics[width=0.92\textheight, height=\textwidth, angle=90]{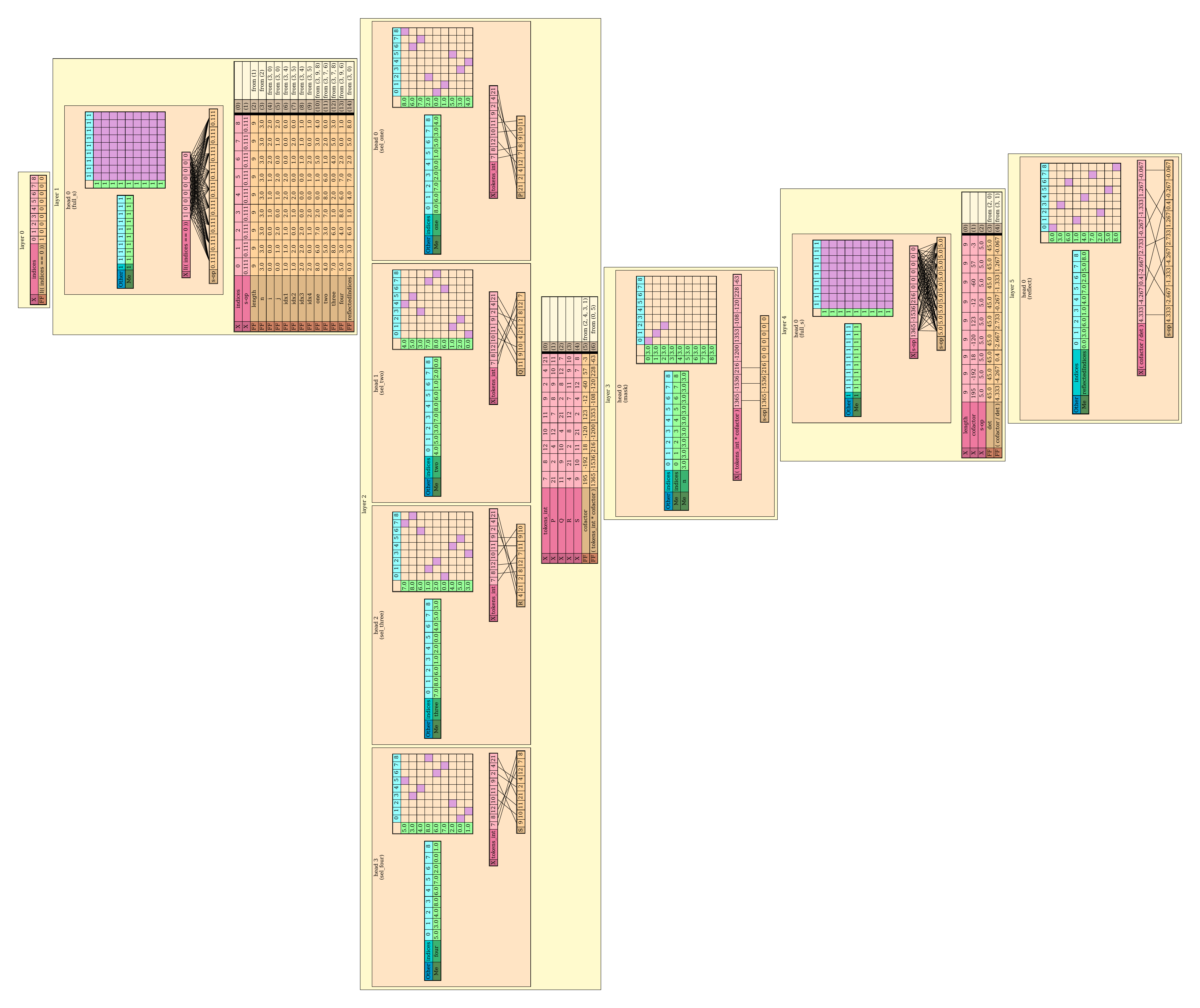}\vspace{-6pt}
    \caption{Execution of $T_{-1}$ on {\small \setlength{\arraycolsep}{3pt} $A = \begin{pmatrix} 7 & 8 & 12\\ 10 & 11 & 9\\ 2 & 4 & 21 \end{pmatrix}$}. View clearly at {\small \url{https://anonymous.4open.science/r/TMA/Inverse.pdf}.}}
    \label{fig:completeinverse}
\end{figure}

The function \lstinline{Cofactor()}, a non-trivial attention layer with four heads representing the agents to pick the four sequences of elements involved to calculate the respective cofactor, and a preceding layer responsible for index calculation. The other one \lstinline{Det()} calculating determinant takes the function \lstinline{Cofactor()} as an argument, thus giving rise to two additional attention layers where the first one attends to the first three indices and the following layer is a trivial one doing the multiplication with \lstinline{length}. However, while inverting a matrix, we may combine some arithmetic operations, mostly taking place in the first layer of attention of the aforementioned functions, which will help us to get a five-layer transformer having a width of four (see \autoref{fig:completeinverse}). The existing implementations of calculating inverse (e.g., \cite{giannou2023looped}) involve Newton's iterative formula. The constant-depth (13-deep, 1-wide) transformers only approximate the solution (up to $T$ steps), and fundamentally, rely on a computational framework that is neither entirely transformer-based (as they use \texttt{for}) nor the classical computational paradigm (as they use transformers). On the other hand, it is important to discuss the constructional challenges with matrix inversion in RASP. Despite producing an exact solution, computing the cofactor of a rank-$(k+1)$ non-singular matrix depends on that of precisely $(k+1)^2$ rank-$k$ matrices. Achieving this inherent \emph{recursive} computation for arbitrary $k$ using a constant-depth architecture such as a transformer is not amenable.

\end{document}